\documentclass[sigconf]{acmart} 

\usepackage{graphicx}
\usepackage{bbm}
\usepackage{caption}
\usepackage{subcaption}
\usepackage{amsmath}
\usepackage{mathtools}
\usepackage{amsthm}
\usepackage{algorithm}
\usepackage{algorithmic}
\usepackage{hyperref}       
\usepackage{url}            
\usepackage{booktabs}       
\usepackage{amsfonts}       
\usepackage{nicefrac}       
\usepackage{microtype}      
\usepackage{xcolor}         

\usepackage{balance} 

\usepackage[utf8]{inputenc}
\usepackage{pgfplots}

\newcommand{\defaultfigwidth}{0.7\linewidth}
\newcommand{\beq}{\begin{equation}}
\newcommand{\eeq}{\end{equation}}

\newtheorem{lemma}{Lemma}

\newtheorem*{dummythm}{Theorem}
\newtheorem*{dummylemma}{Lemma}

\copyrightyear{2025}
\acmYear{2025}
\setcopyright{licensedusgovmixed}\acmConference[KDD '25]{Proceedings of the 31st ACM SIGKDD Conference on Knowledge Discovery and Data Mining V.2}{August 3--7, 2025}{Toronto, ON, Canada}
\acmBooktitle{Proceedings of the 31st ACM SIGKDD Conference on Knowledge Discovery and Data Mining V.2 (KDD '25), August 3--7, 2025, Toronto, ON, Canada}
\acmDOI{10.1145/3711896.3737097}
\acmISBN{979-8-4007-1454-2/2025/08}

\begin{document}

\title{
Quick-Draw Bandits: 
Quickly Optimizing in Nonstationary Environments with Extremely Many Arms
}

\author{Derek Everett}
\authornote{Both authors contributed equally to this research.}
\email{everett\_derek@bah.com}
\orcid{1234-5678-9012}
\affiliation{%
  \institution{Booz Allen Hamilton}
  \city{McLean}
  \state{Virginia}
  \country{USA}
}

\author{Fred Lu}
\authornotemark[1]
\email{lu\_fred@bah.com}
\affiliation{%
  \institution{Booz Allen Hamilton}
  \city{McLean}
  \state{Virginia}
  \country{USA}
}

\author{Edward Raff}
\email{raff\_edward@bah.com}
\affiliation{%
  \institution{Booz Allen Hamilton}
  \city{McLean}
  \state{Virginia}
  \country{USA}
}

\author{Fernando Camacho}
\email{fercamacho@lps.umd.edu}
\affiliation{%
  \institution{Laboratory for Physical Sciences}
  \city{Baltimore}
  \state{Maryland}
  \country{USA}
}

\author{James Holt}
\email{holt@lps.umd.edu}
\affiliation{%
  \institution{Laboratory for Physical Sciences}
  \city{Baltimore}
  \state{Maryland}
  \country{USA}
}

\renewcommand{\shortauthors}{Derek Everett, Fred Lu, Edward Raff, Fernando Camacho, \& James Holt}

\begin{abstract}
Canonical algorithms for multi-armed bandits typically assume a stationary reward environment where the size of the action space (number of arms) is small. More recently developed methods typically relax only one of these assumptions: existing non-stationary bandit policies are designed for a small number of arms, while Lipschitz, linear, and Gaussian process bandit policies are designed to handle a large (or infinite) number of arms in stationary reward environments under constraints on the reward function.
In this manuscript, we propose a novel policy to learn reward environments over a continuous space using Gaussian interpolation.
We show that our method efficiently learns continuous Lipschitz reward functions with $\mathcal{O}^*(\sqrt{T})$ cumulative regret. 
Furthermore, our method naturally extends to non-stationary problems with a simple modification.
We finally demonstrate that our method is computationally favorable (100-10000x faster) and experimentally outperforms sliding Gaussian process policies on datasets with non-stationarity and an extremely large number of arms.
\end{abstract}





\maketitle


\section{Introduction}
\label{sec:introduction}

Generally speaking, multi-armed bandit (MAB) problems seek to characterize, bound, and optimize the performance of sequential decision-making strategies under uncertainty.
The decision space $\mathcal{X}$ consists of $K$ possible arms, where $K$ can be large or infinite.
Each round, a particular arm $x\in\mathcal{X}$ is selected, upon which a stochastic payout is received. 
In the canonical setting of `bandit feedback', the payout of the observed arm is the only information received~\cite{slivkins2019introduction}.
Initially, the payout distributions of all $K$ arms are unknown, and only the stochastic payout of a single arm can be measured at each round.
We denote the (stochastic) observed payout of arm $x$ at round $t$ by $y(x, t)$,
and the expected payout for this arm at this round as $\mu(x, t) \coloneqq \mathbb{E}[  y(x, t) ]$.
The goal of the bandit problem is to converge to the arm(s) with the largest expected reward in an efficient manner.
That is,
if the best arm at round $t$ is $x_t^*$, a regret $r_t \coloneqq \mu(x, t) - \mu(x_t^*, t)$ is incurred each round, and the objective is to minimize the cumulative regret $R_T\coloneqq \sum_{t=1}^T r_t$.

A wide range of approaches have been proposed to solve the MAB problem in cases where the reward environment is stationary and the number of arms is small~\cite{agrawal1995sample,auer2002finite,horvitz1952generalization,kuleshov2014algorithms}.
However, these canonical algorithms break down whenever (1) the reward environment is non-stationary, meaning that reward distributions change over time; or (2) the space is large (e.g. $K\gg 30$), making exhaustive exploration infeasible. 
More recent advances in solving MAB problems have addressed these problems separately, but not together.

In situations where the number $K$ of arms is very large, the bandit problem is only tractable if there is structure or similarities between arms that can be used to derive more optimal policies.
Contextual Bandits assume that at each round, the policy has a \emph{context} -- information that may be relevant to predicting the payoffs. 
The Lipschitz bandit problem, as a particular contextual bandit, 
assumes that arms can be embedded in a metric space such that the mean payout function satisfies a Lipschitz condition:
$|\mu(x_1) - \mu(x_2)| \leq L_x \cdot D(x_1, x_2)$,
where $x_1$ and $x_2$ denote the coordinates of two arms in the metric space, and $D(\cdot, \cdot)$ the metric~\cite{kleinberg2008multi,bubeck2008online}.
The similarity in expected payouts between neighboring arms allows the information of observed payouts to be shared. Therefore, exploration can be targeted to arms that are more likely to yield higher payouts~\cite{kleinberg2019bandits}.

Similarly, in non-stationary settings, the problem is only tractable if the change over time is also constrained: $|\mu(x, t_1) - \mu(x, t_2) | \leq L_t \cdot |t_1 - t_2|$.
To motivate this, assume that the mean rewards of every arm are bounded: $0 \leq \mu \leq 1$.
Let $\tau_s$ denote the time interval between successive observations, 
and $\tau_e$ denote the timescale over which the mean reward environment changes.
This induces a corresponding window size (number of rounds), $
T_w \coloneqq \frac{\tau_e}{\tau_s}$. 
Naturally, policies do not operate optimally in cases where $K \gtrsim T_w$.
In these cases, by the time you can observe every arm at least once, the reward environment has completely changed.

Thus, although methods such as adaptive discretization~\cite{kleinberg2008multi} are designed for Lipschitz reward functions in space,
they will play arms whose expected payouts \emph{were} highest at the earliest times, resulting in sub-optimal arm selection as past payouts become increasingly inaccurate estimates of \emph{current} payouts.
Conversely, methods for non-stationary MAB may degrade when the action space is large.
Yet an environment that varies in both space and time is a setting of real-world interest.
In many applications, we may expect the underlying reward function to vary continuously over a feature space and/or time.
For example, in weather forecasting, factors such as temperature and humidity often change smoothly over location and time.
In medical applications, the effect of a drug may also change smoothly as a function of dosage and duration.

In this manuscript, we propose a MAB policy that can efficiently learn Lipschitz reward functions over a continuous feature space (or with a \textit{very large number of arms}). 
This policy is designed to easily extend to \textit{non-stationary reward environments} with a simple modification.
In particular, the policy can handle extreme situations in which the number of arms $K$ is similar or larger than the observation window, $K \gtrsim T_w$. 
We design the novel policy by adopting similar assumptions as Lipschitz bandit policies but both in time and the metric space. 
Our model builds an index function that models both sources of uncertainty regarding the mean payout function and allows the information about past arm payouts to be used for a more optimal and structured exploration. 
While the underlying probabilistic model is similar to prior work on Gaussian process (GP) bandits~\cite{srinivas2009gaussian}, it has key differences that allow for efficient scaling with the number of iterations and for non-stationary reward functions.
We prove that our method obtains $\mathcal{O}^*(\sqrt{T})$ cumulative regret on stationary Lipschitz payouts, which is comparable to GP policies.

Our policy is simple to implement and is validated in both simulated and real-world environments, providing statistically significant improvements. 
In simulated experiments on a challenging nonstationary reward environment, we demonstrate that our Quick-Draw bandit policy handles noise in the reward function and exploits correlations in time better than existing methods. 
Moreover, our policy achieves a $65\%$ relative improvement compared to the next-best approach on a real-world advertising dataset where policies aim to maximize click-through rate.

\section{Review of Existing Methods}
\label{sec:review}

The conceptually simplest bandit policies are non-adaptive. They pose conditions for exploration \textit{a priori} -- before any rewards are observed. Examples include the Explore-First and $\epsilon$-Greedy policies~\cite{slivkins2019introduction}. 
Adaptive bandit policies are usually assumed to outperform non-adaptive policies. 
Generally, adaptive policies operate by modeling the mean reward distribution over each arm given the observed payouts. 
Given these distributions, 
the policy typically either 
draws random samples from the distributions or deterministically selects the arm which has a maximal score or `index'.
Thompson sampling~\cite{thompson1933likelihood, thompson1935theory} and softmax methods~\cite{sutton1999reinforcement} are of the first type. On the other hand, the multitude of Upper Confidence Bound (UCB)~\cite{lai1985asymptotically,agrawal1995sample,auer2002finite} strategies are of the second type. UCB strategies follow a paradigm of `optimism under uncertainty'~\cite{lai1985asymptotically}.

For UCB strategies, 
each arm is assigned an index
that corresponds to an approximation of an $X\%$ upper confidence bound on the mean payout. In most cases, this interval is constructed according to
$
\rm{UCB} \equiv \hat{\mu} + \hat{\Delta},
$
where $\hat{\mu}$ is an estimate of the mean payout for the arm, potentially estimated via a weighted or sliding average over past observations, and
$\hat{\Delta}$, sometimes called a `radius', is a model of the uncertainty on this average. 
Many UCB policies have a radius for the $i^{\rm th}$ arm that can be written in the form 
$
\hat{\Delta}(i, s) = \frac{B}{\sqrt{\Tilde{n}(i, s)}}
$
where $B$ is a bound on the rewards, and $\Tilde{n}(i, s)$ is a function that depends on the number of times arm $i$ has been played within a time window, discounting factors, and the time horizon.
As an example, the UCB-1 policy fixes
$
\hat{\Delta}(i, s) = B\sqrt{ \frac{\xi \log T }{n(i, s)} }
$
where $\xi>0$ is a constant controlling the size of the confidence interval~\cite{slivkins2019introduction}.
The Discounted UCB~\cite{kocsis2006discounted} and Sliding-Window UCB~\cite{garivier2008upper} are two Upper Confidence Bound policies designed to handle non-stationary reward environments.
However, these policies do not have mandatory exploitation phases
and will continue exploring until every arm has been observed at least once.
This is untenable for the extreme problems that we tackle, namely where the number of arms is as large as the time Horizon, $K \approx T$. These policies are equivalent to random exploration in these situations.   

On the other hand, the `Restless bandit' policy, as described in ~\cite{slivkins2008adapting}, is a policy designed for non-stationary reward environments, and has mandatory exploitation rounds. 

This policy calculates a `suspicion' index for each arm, 
which is given by

$
\rm{suspicion}_j = y_j + \sigma_r \sqrt{\Delta t} - y_l,
$
where $y_j$ was the last observed pay for the arm, 
$y_l$ is the payout of the current `leader' arm (which is updated at each round), 
$\sigma_r$ is a parameter and 
$\Delta t$ is the time between the current round and the last round the arm was observed.
The leader arm is played often, at least every other round. 
Otherwise, the suspicions of all arms are calculated, and if any arm's last observed suspicion exceeds zero 
it enters a subset of candidate arms for exploration.
This policy was designed for arms whose payouts undergo independent Brownian motion/random walks.

Other works have also studied problems with non-stationary payouts that change smoothly, rather than abruptly, in time; however, they are restricted to the limit of a small number of arms (e.g., a \emph{two-armed bandit} was studied in ~\cite{jia2023smooth}).
Non-stationary linear bandit policies are the primary methods that have been developed thus far that can be used even in the presence of a large number of arms~\cite{zhao2020simple,russac2019weighted}.
However, linear bandit policies assume that the mean reward function is an (unknown) linear function of the arm index, making them inapplicable for reward environments that are not well-described by linear functions.
In addition, Bandit Convex Optimization policies for non-stationary functions have been developed, but these methods assume that each point in time the reward function is an (unknown) \emph{convex function}~\cite{zhao2021bandit}. In this manuscript, we focus on optimizing the regret for unknown non-convex and non-linear payout functions.

On the other hand, stationary rewards with a large or infinite number of arms is studied in the continuum-armed bandit and Lipschitz bandit literature~\cite{kleinberg2004nearly,bubeck2008online,bubeck2011lipschitz,auer2007improved}.
These algorithms are often discrete, such as the zooming algorithm~\cite{kleinberg2008multi}.
Meanwhile, continuous probabilistic models of the action space using Gaussian processes (GP) have recently shown a lot of interest~\cite{srinivas2009gaussian,djolonga2013high,vakili2021optimal,bogunovic2021misspecified}.
These methods model the reward distribution at a collection of points as a multivariate Gaussian distribution with mean and covariance defined as a kernel function of observed points. 
While our policy also models the reward space as a Normal likelihood, our modeling approach can instead be viewed as a kernel interpolation algorithm.
This approach has the advantage of linear computational complexity over observations and is significantly faster and cheaper than Gaussian processes.
Additionally, our work models the time dependence of the reward function explicitly, which is not readily handled in the GP bandit works.

In our work, we will explicitly contrast our formulation with Gaussian processes.
Efforts have also been made to extend Gaussian Process bandit policies to non-stationary reward environments, including \cite{deng2022weighted} which applies a time-dependent sample-weighting.
In particular, ~\cite{zhou2021no} extended the GP-UCB policy to allow for nonstationarity by using a sliding window, called the SW-GP-UCB policy. We include this latter policy as a baseline of comparison throughout. 

Existing literature that empirically compares the performance of MAB algorithms (rather than proving bounds based on worst-case assumptions) is limited, with a notable exception in~\cite{kuleshov2014algorithms}. However, this study was performed in the canonical stochastic bandit setting of stationary payouts and fewer arms.

\section{Our Novel Quick-Draw Bandit Policy}
\label{sec:methods}

We assume that the action space consists of arms labeled $x$ in a metric space $\mathcal{X}$,
with normalized distance metric $D(x_i, x_j) \leq 1$,
and that the mean payout function $\mu(x,t)$ is a continuous Lipschitz function of both $x$ and time $t$.
Without loss of generality, we bound the payout $\mu(x, t)\in[0, 1]$.
For convenience, given a sequence of $T$ observed values, we will refer to the selected arms, their corresponding means, and observed rewards as $\{x_s\}_{s=1}^T$, $\{\mu_s\}_{s=1}^T$, $\{y_s\}_{s=1}^T$ respectively. 

\subsection{Stationary reward setting}
First, we consider the stationary case, where the payout $\mu(x)$ is solely a function of space. 
Our model then estimates the payout at $x$ given a previously observed payout $\mathcal{D}_s\coloneqq(x_s, y_s)$ by assuming a conditional Normal likelihood: $\mathcal{P}(\mu(x) | \mathcal{D}_s) = \mathcal{N}(y_s, \hat{\sigma}_s(x))$,

with the uncertainty $\hat{\sigma}(x)$ modeled as a function of the distance of any $x$ from the previously observed $x_s$:
\beq
\hat{\sigma}^2_s(x) \equiv \rho^2 + \left(\frac{D(x, x_s)}{\ell_x} \right)^2
\label{eqn:sigma_hat_stationary}
\eeq

Then, we assume that the joint likelihood over the past observations of each round up to the current round $T$ is a product, 
$
\mathcal{P}(\mu_T(x) | \mathcal{D}_T) = \prod^T_{s=1} \mathcal{P}(\mu_s(x) | \mathcal{D}_s).
$
Conveniently, the product of $T$ Gaussian densities is itself a Gaussian density whose mean and variance have closed-form expressions~\cite{bromiley2003products}. Explicitly, 
$
\mathcal{P}(\mu_T(x) | \mathcal{D}_T) = \mathcal{S}_T \cdot \mathcal{N}(\hat{\mu}_T(x), \hat{\Sigma}^2_T(x)),
$
where the variance $\hat\Sigma_T^2$
and mean $\hat\mu_T$ are defined as:
\begin{equation} \label{eqn:mu-sigma}
\hat{\Sigma}^2_T(x) = \left[ \sum^T_{s=1} \frac{1}{\hat{\sigma}_s^2(x)} \right]^{-1},
\hspace{1em}
\hat{\mu}_T(x) = \left[ \sum^T_{s=1} \frac{y_s}{\hat{\sigma}_s^2(x)} \right] \hat{\Sigma}_T^2(x)
\end{equation}\noindent

The normalization $\mathcal{S}_T$ also has a closed-form expression,
but is not necessary given that the bounds of any credible interval are independent of the overall normalization of the distribution. 
For each of the $K$ arms, we compute the joint likelihood and credible interval; 
the Quick-Draw index for each arm is given by
$
\rm{UCB}_{\rm QD} = \min( \hat{\mu}_T + \gamma_{T+1}\hat{\Sigma}_T, 1 ).
$
where $\gamma_{T+1}$ is a scaling constant. If $\gamma_{T+1}=2$ this corresponds to an approximation of the $95\%$ upper credible bound, with a ceiling given by $1$ because we have assumed the mean payouts satisfy $0 \leq \mu(x, t) \leq 1$. 
During each round, the Quick-Draw policy plays the arm that maximizes 
the $\rm{UCB}_{\rm QD}$ index (see Algorithm~\ref{alg:dl_mab}).
\begin{algorithm}[H]
\caption{Quick-Draw Multi-Armed Bandit Policy}
\label{alg:dl_mab}
\begin{algorithmic}
\STATE \textbf{Input:} Number of arms $K$, time horizon $T$, parameters $\ell_x$, $\ell_t$,$\gamma$
\STATE \textbf{Initialize:} Observations $\mathcal{D}_t = \emptyset$

\FOR{each round $t = 1$ to $T$}
    \FOR{each arm $k = 1$ to $K$}
        \STATE Compute $\hat{\Sigma}_T$ and $\hat{\mu}_T$ using Equations $\eqref{eqn:mu-sigma}$
        \STATE Compute $\mathrm{UCB}_{k,t} = \min( \hat{\mu}_T+\gamma_{T+1}\hat{\Sigma}_T, 1 )$
    \ENDFOR
    \STATE Select arm $a_t = \arg\max_k \rm{UCB}_{k,t}$
    \STATE Observe payout $y_{a_t, t}$
    \STATE Update $\mathcal{D}_{t+1} = \mathcal{D}_{t} \cup \{ (a_t, y_{a_t, t}) \}$
\ENDFOR
\end{algorithmic}
\end{algorithm}

\subsection{Non-stationary reward setting}

In a non-stationary reward environment, we assume the mean payout function $\mu(x, t)$ is now a continuous Lipschitz function of both time $t$ and space $x$.
As before, we bound $\mu(x, t) \in [0, 1]$.

Now our model estimates the payout as \\$\mathcal{P}(\mu(x, t) | \mathcal{D}_s) = \mathcal{N}(y_s, \hat{\sigma}_s(x, t))$.
One of the clear advantages of our formulation is that we can immediately handle non-stationary environments simply by adding a term to the uncertainty function:
\beq
\hat{\sigma}^2_s(x,t) \equiv \rho^2 + \left(\frac{D(x, x_s)}{\ell_x} \right)^2 + \left(\frac{t - t_s}{\ell_t}\right)^2
\label{eqn:sigma_hat}
\eeq
Then the estimated mean $\hat\mu_T(x, t)$ and variance $\hat\Sigma_T(x, t)$ are functions of space and time but are otherwise identical to Eq.~\ref{eqn:mu-sigma}.

\subsection{Interpreting the hyperparameters}

The Quick-Draw policy parameters $\rho$, $\ell_x$ and $\ell_t$ are all that remain for the policy to be determined.
If we view the precision function $\nu_s(x, t) \coloneqq 1/\hat\sigma_t^2(x, t)$ as a similarity function between points,
then our mean function can be seen as a weighted average kernel interpolation where the weights are determined as $\nu_s / \sum_t \nu_t$.
In this view, then $\ell_x$ and $\ell_t$ are bandwidth parameters -- the smaller they are, the more local the weighting becomes.
Finally, $\rho^2$ models the irreducible sampling error of the observed payout $y_t$ around the underlying mean value $\mu(x, t)$.

In practice,
while the parameters can be optimized empirically,
we found that their performance was only sensitive to their order of magnitude.
Because the distance over $\mathcal{X}$ is bounded by 1 and distance over time is a bounded function of $\tau_s$,
these can be used to motivate the appropriate order of magnitude of $\ell_x$ and $\ell_t$.
Furthermore, our theoretical results can give a sense of the effect of the hyperparameters on the cumulative regret function. 
In our experiments, we note the values of $\ell_x$ and $\ell_t$ as well as demonstrate the performance and (in)sensitivity of the policy as the order-of-magnitudes are varied.

Additionally, in Fig.~\ref{fig:ctr_vs_lt} we display the performance of the Quick-Draw policy on the Open Bandit dataset as the policy parameter varies. We see that the performance of the policy is stable, and its performance has only minor sensitivity to $\ell_t$. In our comparison experiments, we used the default values of $\ell_x = \ell_t = 1$, and used $\rho^2 = 10^{-7}$ for numerical stability. Although the policy demonstrates insensitivity (robustness) to the hyperparameter selection, if a different task requires tuning, one can simply do a grid-search to maximize the cumulative reward.

\subsection{Comparison with Gaussian Process in stationary reward setting}

\begin{figure}[!h]
    \centering
    \includegraphics[width=\defaultfigwidth]{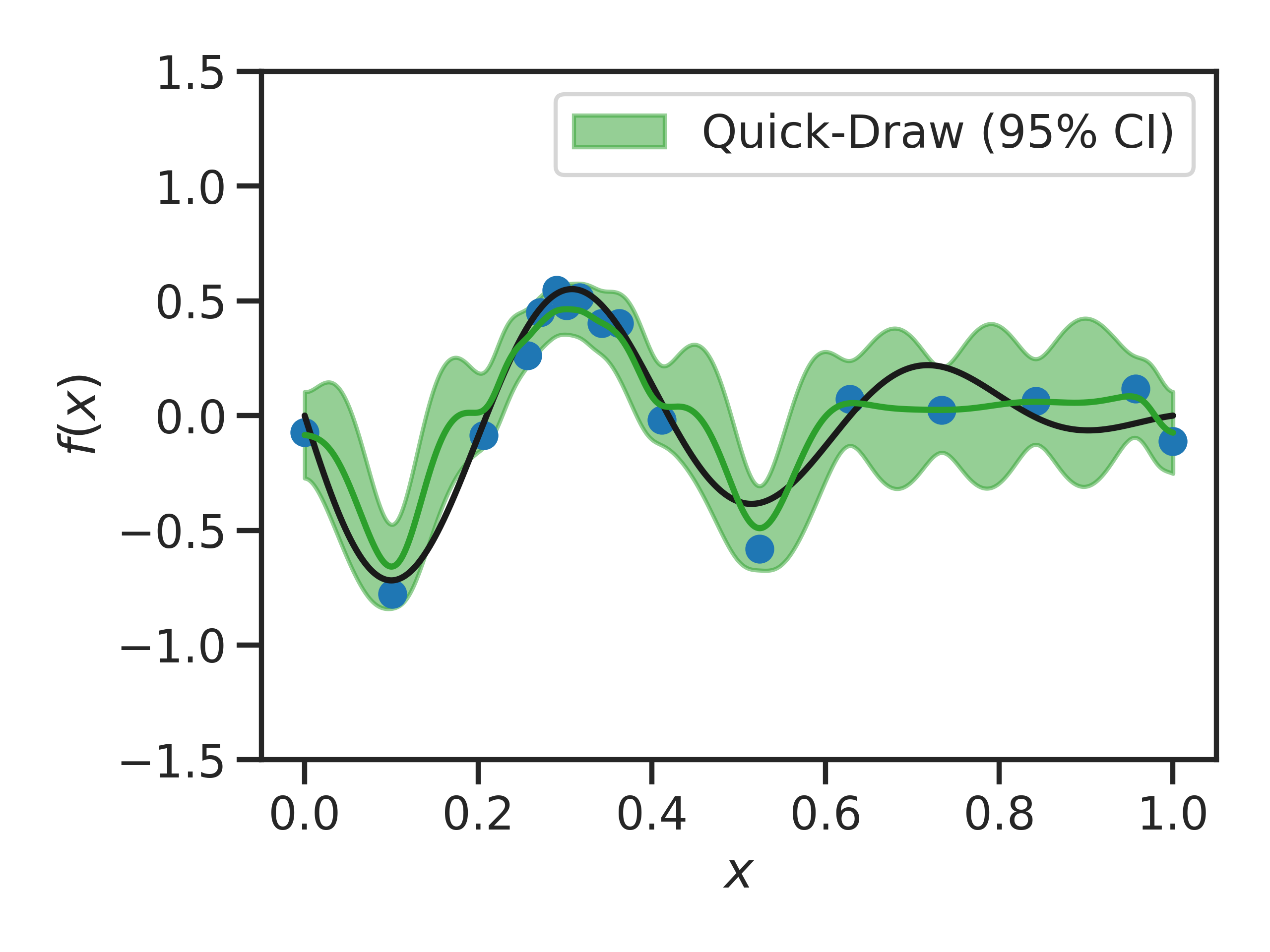} \\[-0.6cm] 
    \includegraphics[width=\defaultfigwidth]{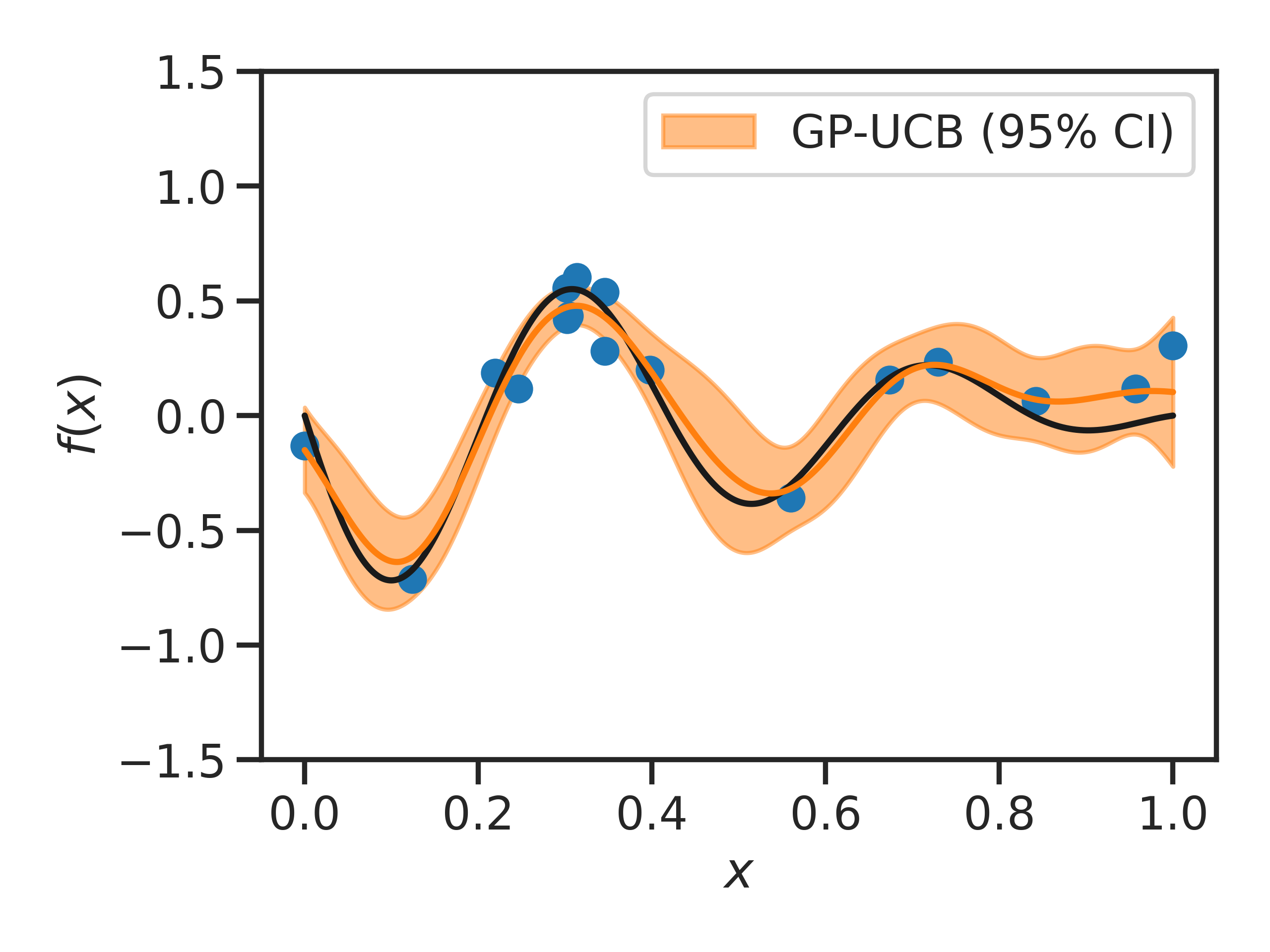} \\[-0.6cm]
    \includegraphics[width=\defaultfigwidth]{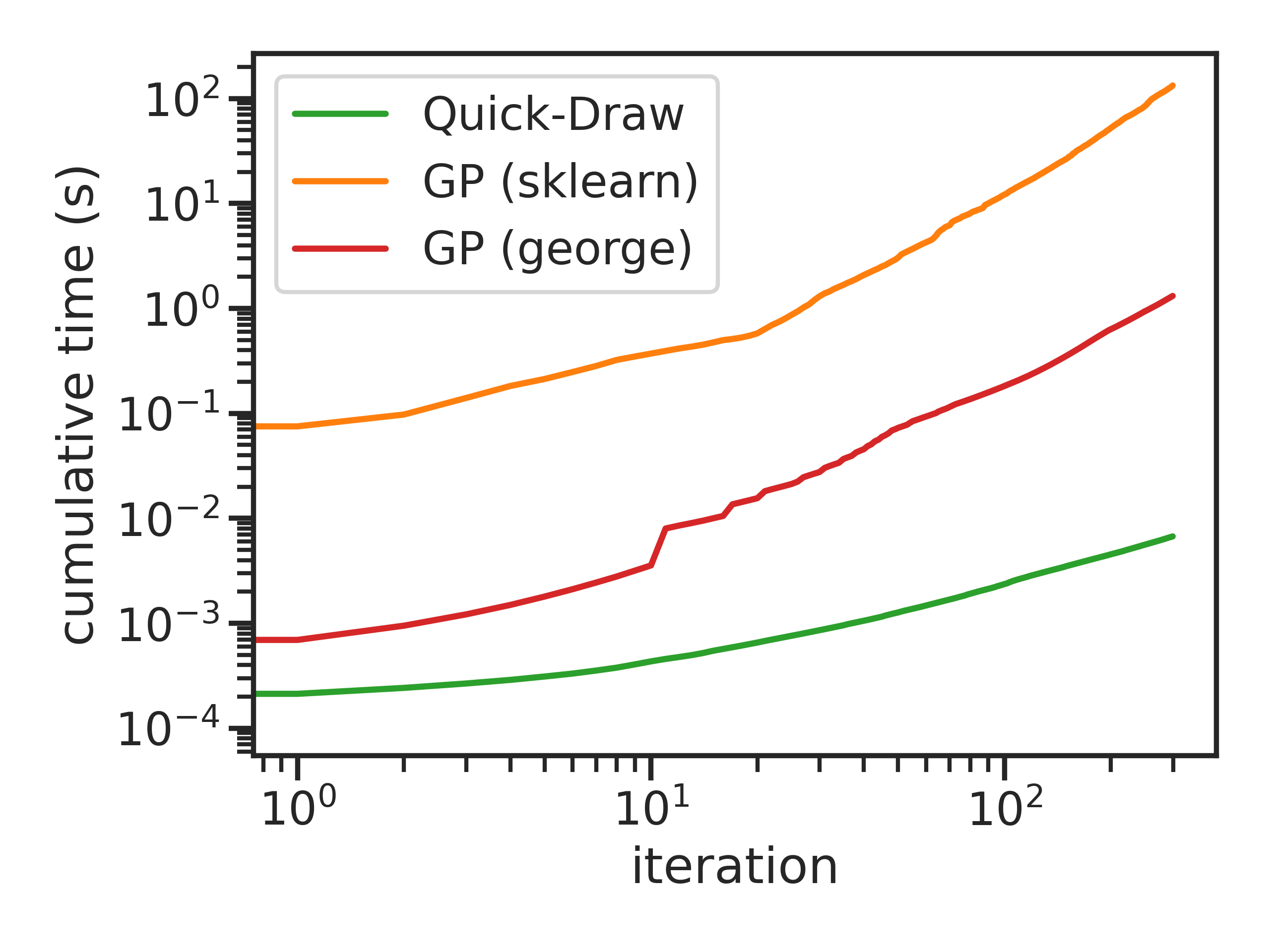}
    \caption{Quick-Draw (top) and GP-UCB (middle) bandits converge toward the maximum payout location after 15 iterations. However, the cumulative runtime for Quick-Draw is extremely fast, in comparison to GP methods (bottom).}
    \label{fig:runtime}
\end{figure}

Like our method, the GP bandit also learns a nonparametric mean as a weighted function of past observed rewards.
In particular, we can define a kernel corresponding to our stationary similarity function $k(x_1, x_2) \coloneqq 1/\hat\sigma_1^2(x_2) = 1/\hat\sigma_2^2(x_1)$.
This yields a mean function of form $\hat\mu_T(x) = \frac{ \sum_t k(x, x_t) y_t}{\sum_t k(x, x_t)}$. This shows that our mean function is an interpolation of observed rewards, with the form of the Nadaraya-Watson kernel estimator~\cite{nadaraya1964estimating}.
In contrast, the GP mean function takes form $\hat\mu_{T}^{GP}(x) = k_T(x)^\top (K_T + \rho^2I)^{-1}y_T$,
where $k_T$ is a vector of the kernel weights $k(x, x_i)$ and $K_T$ is the kernel similarity matrix.

While the GP mean function can take a broader range of values, it requires inversion of $K_T$, which is inherently $\mathcal{O}(T^3)$ complexity. A similar observation holds for the variance function. 
In contrast, a brute-force implementation of our method is $\mathcal{O}(T)$; moreover, each update is $\mathcal{O}(1)$ if we cache each previous similarity function $\nu_s$. This gives our method an inherent computational advantage which enables it to scale to much larger datasets than are feasible for GPs.

While exact GPs scale as $\mathcal{O}(T^3)$, low-rank approximate GPs can scale as $\mathcal{O}(T * m^2)$, where $T$ is the number of observations and $m$ a parameter, bringing them closer to the $\mathcal{O}(T)$ performance of the QuickDraw policy. While this can improve runtime, this can also hurt the accuracy of the GP approach. Since our experiments show exact GPs having worse regret than our QuickDraw policy, we do not consider the approximate GP policies.

We next compare the learned bandit functions from Quick-Draw and GP on a synthetic dataset with known reward function, after tuning the hyperparameters (see Fig.~\ref{fig:runtime}).
After 15 iterations both algorithms have identified the region of high reward,
though in practice the GP method was more sensitive to the hyperparameters and thus harder to tune.
However, the runtimes of the methods are drastically different.
Comparing our method with two separate implementations of GPs (scikit-learn~\cite{pedregosa2011scikit} and george~\cite{ambikasaran2015fast}),
our bandit algorithm is multiple orders of magnitude faster.

\section{Regret Bound}  \label{sec:bound}

In this section, we formalize our theoretical framework.
As before, we assume the payout function $\mu(x,t)$ has a Lipschitz constraint over time and space with constant $L$.
Furthermore, we suppose that the reward of each arm $x\in\mathcal{X}$ follows a symmetric distribution around the mean function (e.g. Normal distribution $\mathcal{N}(\mu(x, t), \tau^2)$.

In the following, we derive the regret for Quick-Draw bandits in space with a stationary reward over time.
Considering the stationary regime permits us to situate our method in the continuous bandit literature in comparison with Gaussian process bandits.
Because of the high number of arms, there is a low chance of the same arm being sampled more than once, at least in the beginning.
Therefore, we must make inferences using neighboring observations, which is fundamentally more challenging than traditional UCB algorithms. Our first result is a concentration inequality which relates the UCB scaling parameter $\gamma_{T}$ used in Algorithm~\ref{alg:dl_mab} to the regret on any point.

For simplicity, we define the constant 
$C_1 \coloneqq \frac{\sqrt{\rho^2 + 1/\ell_x^2}}{\rho^2}$
as a function of the hyperparameters.
Also we denote $[T] \coloneqq \{1, \ldots, T\}$.




\begin{theorem} \label{thm:concentration}
Suppose the Quick-Draw algorithm is run for up to $T_{\max}$ rounds,
and set hyperparameter $\ell_x$ small enough so that $\ell_x\leq 1/\sqrt{T_{\max} - \rho^2}$.
For any $\delta > 0$, set $\gamma_T \coloneqq 2 L + 4 C_1 \ln^2(2T^2/\delta)$.
Then for all $x\in\mathcal{X}$ and all $1\leq T\leq T_{\max}$, 
    \begin{equation}
        | \mu(x) - \hat\mu_T(x)| \leq \gamma_{T+1}\hat\Sigma_T
    \end{equation}
    holds with probability at least $1 - \delta$.
\end{theorem}

Using this concentration bound, we can directly quantify the regret of our UCB policy per iteration.

\begin{lemma}  \label{lem:ucb_regret_per_round}
Choose any $t\leq T_{\max}$. Under the condition $\{ |\mu(x) - \hat\mu_t(x) | \leq \gamma_{t+1} \hat\Sigma_t(x) \}$, for all $x\in\mathcal{X}$, then the regret at round $t$
satisfies
$$ r_t \leq 2 \gamma_{t+1} \hat\Sigma_t(x_t) $$
where $x_t$ is the point selected by the UCB policy.
\end{lemma}

Finally we will combine the results to obtain our regret bound.

\begin{theorem} \label{thm:total_regret}
    Let $\delta > 0$. Run the Quick-Draw algorithm for $T_{\max}$ iterations where $T_{\max}$ can be arbitrarily large.
    If $\gamma_T$ and $\ell_x$ are set as before,
    then with probability $1-\delta$,
    the cumulative regret is
    $$R_T \leq 4CL + 8C^2\sqrt{T}\ln^2(2T^2/\delta)$$
    where $C\coloneqq \sqrt{\rho^2 + 1/\ell_x^2}$.
    
    This gives the asymptotic regret
$$R_T = \mathcal{O}\Big(C^2 \sqrt{T} \ln^2(2T^2/\delta)\Big)$$
\end{theorem}

This is a comparable order of regret as the classic GP bandit over correctly-specified GP reward functions or constrained RKHS norm~\cite{srinivas2009gaussian}.
Our result holds over Lipschitz reward functions which was outside the scope of the original GP work, and yet in general more intuitive to examine in real-world settings.

\section{Experiments on Simulated Data}
\label{sec:experiments}

\begin{figure}[!h]
\centering
\begin{subfigure}{.5\linewidth}
  \centering
  \includegraphics[width=\linewidth]{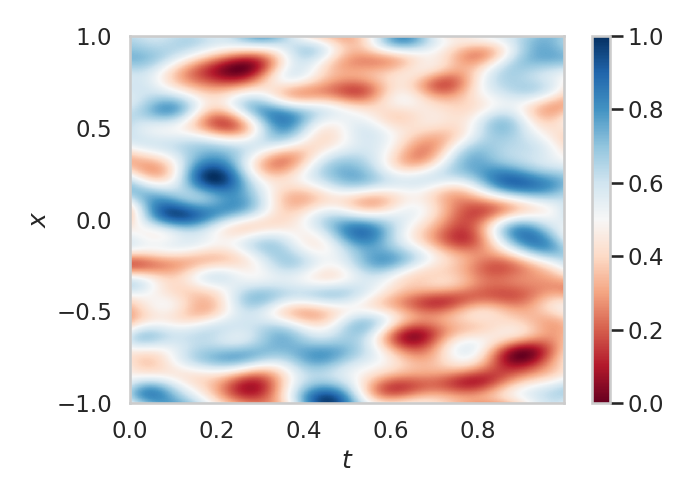}
\end{subfigure}%
\hspace{-0.1\linewidth}
\begin{subfigure}{.5\linewidth}
  \centering
  \includegraphics[width=\linewidth]{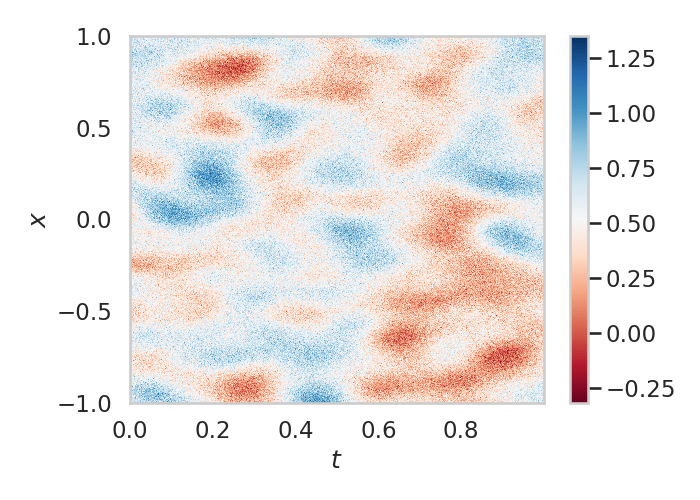}
\end{subfigure}
\vspace{-0.05cm}
\begin{subfigure}{.5\linewidth}
  \centering
  \includegraphics[width=\linewidth]{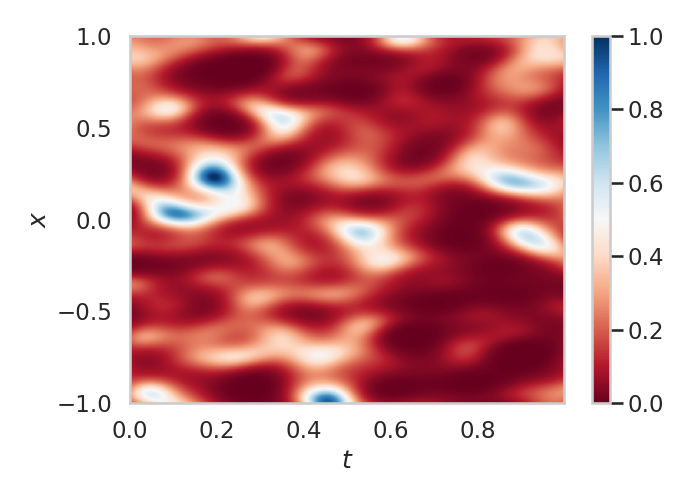}
\end{subfigure}%
\hspace{-0.1\linewidth}
\begin{subfigure}{.5\linewidth}
  \centering
  \includegraphics[width=\linewidth]{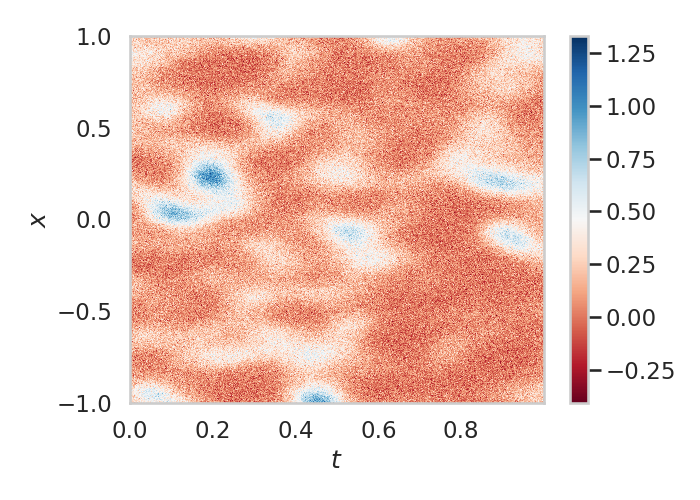}
\end{subfigure}
\caption{A sample of the mean (left) and noisy (right) payout functions $\mu(x, t)$ and $y(x, t)$ that may be observed by the bandit at each time. The payout fields are generated with spatial and temporal correlation parameters $\rho_x = \rho_t = 0.1$, white noise parameter $\sigma_{\rm noise}=0.1$, and sharpness parameters $\alpha=1$ (top row) and $\alpha=3$ (bottom row).}
\label{fig:fields_alpha}
\end{figure}

We developed an experimental test bed as follows: 
In each experiment, we sample a Gaussian random field using \emph{GSTools} 1.5~\cite{muller2022gstools}. The field is generated with correlation length $\rho_x$ over the metric space and temporal correlation length $\rho_t$. 
The field values are rescaled to the interval $[0, 1]$, and then the field values are raised to a power $\alpha$,
and this defines the mean payout function $\mu(x, t)$. 
When the power $\alpha>1$, the regions of high mean payout become more sharply defined, increasing the difficulty of the problem (the `hills' become `peaks'). 
Then, to this mean field is added white noise with standard deviation $\sigma_{\rm noise}$. 
An example of a random mean payout field $\mu(x, t)$ with scales $\rho_t=10^{-1}$ and $\rho_x=10^{-1}$ and power $\alpha=1$ is shown next to a noisy sample counterpart $y(x, t)$ with $\sigma_{\rm noise}=0.1$ in Fig.~\ref{fig:fields_alpha} (top). 
Similarly, an example of the random mean payout field with the same scales and random seed, but raised to the power $\alpha=3$ is also shown in Fig.~\ref{fig:fields_alpha} (bottom). 

We fix the rounds to be at regularly spaced intervals of size $\tau_s = 10^{-3}$ and choose $K=10^3$ bandit arms (at a spacing in $x$ of $2 \cdot 10^{-3}$). 
Additionally, in each experiment, the first $100$ rounds are fixed to a `warm-up' policy of random sampling.
This allows each policy to start with 100 past observations.  
We compare policies by calculating the cumulative regret starting after the `warm-up' period and dividing by the number of rounds.  
The bandit problem over the observed values of the field $y(x, t)$, which seeks to maximize the payout and minimize regret, 
becomes increasingly difficult when either of the scales $\rho_x$ or $\rho_t$ are decreased. This is because, at a fixed rate of time samples, the information about neighboring observations in space and time becomes less correlated with expected payouts at the current location and time. Similarly, the bandit problem becomes more difficult when the power $\alpha$ is increased, because the difference between optimal and sub-optimal arms is magnified/sharpened. Finally, the difficulty is increased when $\sigma_{\rm noise}$ is increased, because the bandit observes noisier estimates of the mean payout. 

Therefore, we have performed experiments that vary each of these parameters. For every parameter setting, we sample 100 random payout functions, playing each policy on the payout function and recording its cumulative regret. In each plot is shown the average over these experiments. 
The baseline parameters are fixed to be 
$\rho_x = \rho_t = 0.1$, $\alpha=1$ and $\sigma_{\rm noise} = 0$. Then, we varied each of these parameters individually and observed the performance of four bandit algorithms as well as a random selection policy to provide a quantitative baseline. 
The four algorithms shown in the plots are 
\begin{enumerate}
    \item the sliding window Gaussian Process UCB policy (SW-GP-UCB) described in ~\cite{zhou2021no} with $T_w = 100$. At every round, the GP's parameters are optimized on only the observations within the window.
    \item $\epsilon$-greedy on a sliding window $T_w = 100$: $\epsilon=0.1$ is the probability of randomly choosing an arm, and otherwise the arm with the highest observed pay within the window is played. 
    \item restless bandit policy: As described in ~\cite{slivkins2008adapting}, this policy maintains a suspicion index for each arm which is given by $y_s + \sigma_r \sqrt{\Delta t} - y_l$, where $t_s$ was the last observed pay for the arm, $\sigma_r$ is a parameter, $\Delta t$ is the time between the current round and the last round the arm was observed, and $y_l$ is the payout of the leader arm. 
    \item our novel Quick-Draw bandit policy
\end{enumerate}

Although it is not shown in any figures, we also implemented and tested the discounted sliding-UCB algorithm~\cite{garivier2008upper}. However, the algorithm specifies that the first $s=1, \cdots, K$ rounds are played by choosing arm $s$. 
But our experimental test bed has a total horizon $T = K$. Furthermore, in all of our experiments, we have that $T_w = \tau_s / \rho_t \leq K$. 
Even if we were to adapt the algorithm, we can not observe all $K=10^3$ arms within any window of time for which the values of observed payouts remain correlated/informative. 
Therefore, this policy would be equivalent to random sampling in this situation.

\begin{figure}[!h]
\centering
\begin{subfigure}{.5\linewidth}
  \centering
  \includegraphics[width=\linewidth]{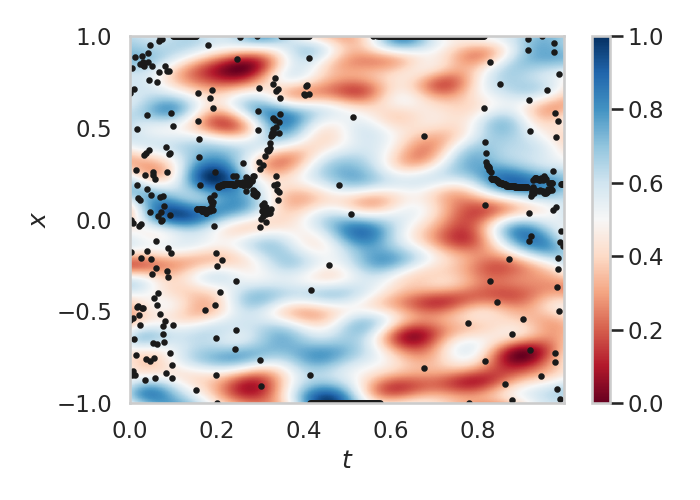}
\end{subfigure}%
\hspace{-0.1\linewidth}
\begin{subfigure}{.5\linewidth}
  \centering
  \includegraphics[width=\linewidth]{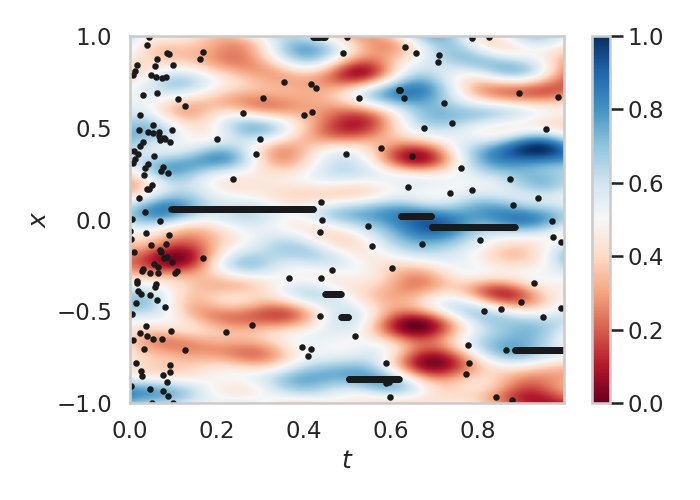}
\end{subfigure}
\vspace{-0.05cm}
\begin{subfigure}{.5\linewidth}
  \centering
  \includegraphics[width=\linewidth]{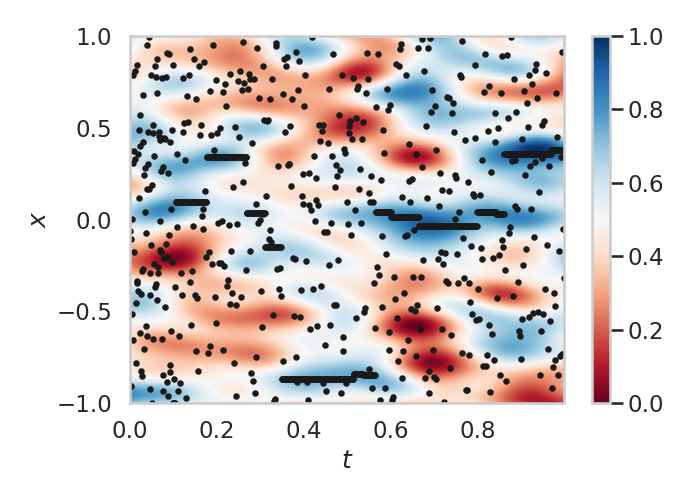}
\end{subfigure}%
\hspace{-0.1\linewidth}
\begin{subfigure}{.5\linewidth}
  \centering
  \includegraphics[width=\linewidth]{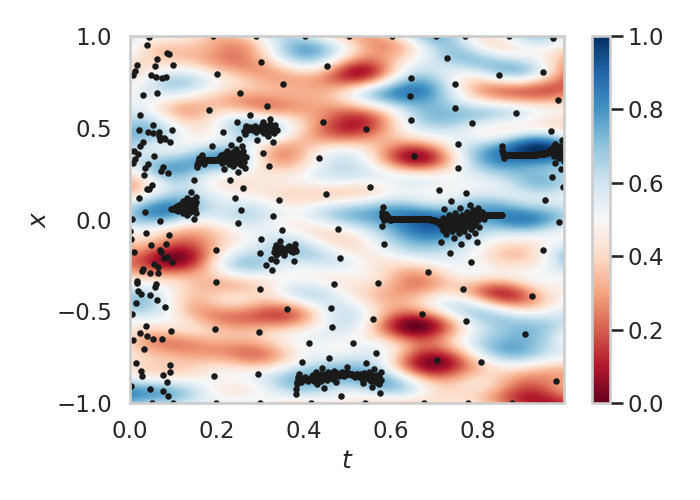}
\end{subfigure}
\caption{The behavior of four bandit policies displayed on top of the same mean payout function. 
The arms chosen by each policy are shown as black dots. 
The upper left displays the SW-GP-UCB policy. The upper right displays a sliding $\epsilon$-greedy policy with $\epsilon=0.1$ and a window size $T_w=100$. The lower left displays the restless bandit policy. The lower right displays our novel Quick-Draw policy with $l_x = l_t = 1$. The first $100$ rounds ($0 \leq t \leq 0.1$) for all policies are fixed to random sampling.}
\label{fig:mean_field_policies}
\end{figure}

Before quantifying the performance of the policies on ensembles of experiments, we plot and compare their behavior on a single problem instance (fixed seed). As a reminder, the first $100$ rounds $0 \leq t \leq 0.1$ are fixed to be a random sampling `warm-up' stage for every policy. 
These are displayed in Fig.~\ref{fig:mean_field_policies}.
We see that the sliding $\epsilon$-greedy policy performs favorably, but continues playing arms after their mean payout has considerably decreased. 
The restless bandit and Quick-Draw policies incorporate and model the time dependence of the expected payouts. Therefore, they more quickly adapt to a decreasing payout of previously exploited arms. 
However, the restless bandit policy does not incorporate prior information regarding the payout distributions in the metric space. 
The SW-GP-UCB policy can be seen to engage in `edge-seeking' behavior too often, even when payouts on the edges of the space are highly suboptimal. This is caused by compounding failures of the time-independent Gaussian Process model of the time-dependent reward function within the observation window. This will be illustrated further in the plots below. 
Finally, The Quick-Draw policy, which models both the dependence in time and over the metric space,
has a more structured and optimized selection of the arms.

\begin{figure}[!h]
\centering
  \centering
  \includegraphics[width=\defaultfigwidth]{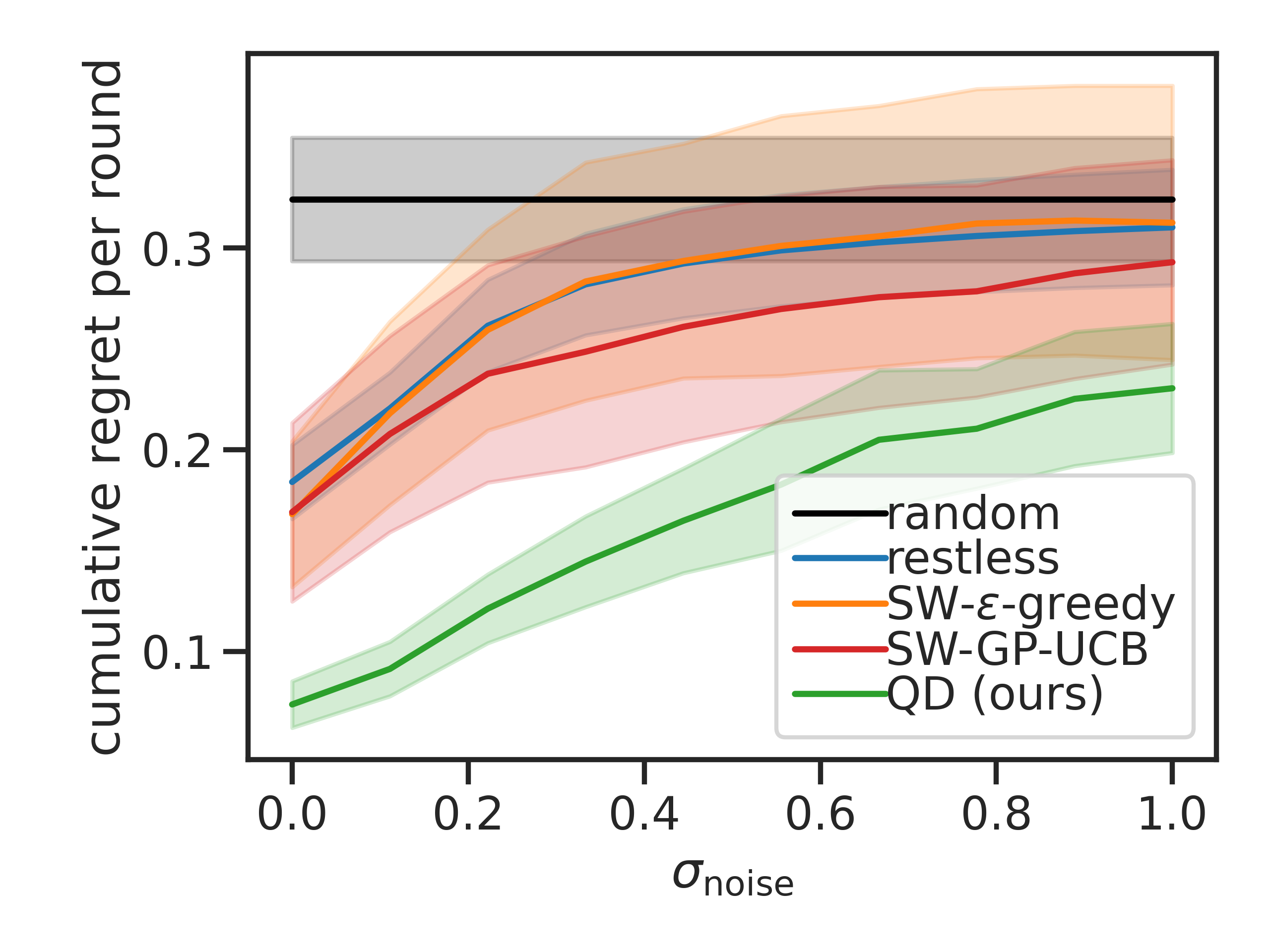}
  \caption{Comparison of the cumulative regret per round for each bandit policy as the amount of white noise $\sigma_{\rm noise}$ is varied.}
  \label{fig:results_noise}
\end{figure}

\begin{figure}[!h]
\centering
  \centering
  \includegraphics[width=\defaultfigwidth]{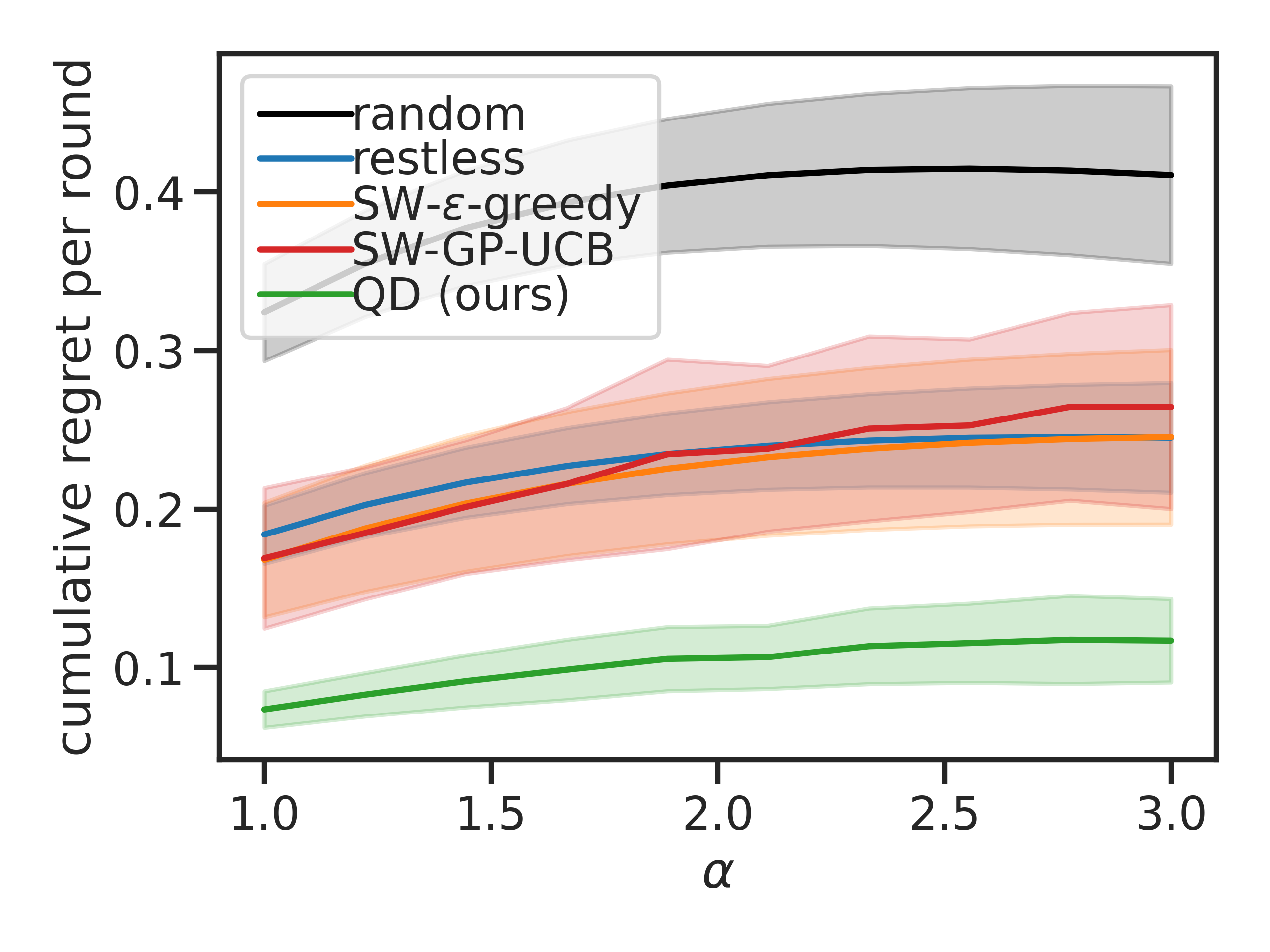}
  \caption{Comparison of the cumulative regret per round for each bandit policy as the sharpness of the mean payout function $\alpha$ is varied.}
  \label{fig:results_field_power}
\end{figure}

In Fig.~\ref{fig:results_noise} we show the results of the bandit policies varying the white noise level $\sigma_{\rm noise}$ present in the payout function $\mu(x, t)$. We see that the Quick-Draw bandit policy systematically outperforms all baselines. 
We also see that the sliding-window $\epsilon$-greedy policy is significantly better than random sampling, and performs very similarly to both the restless bandit and SW-GP-UCB policies. ~\cite{kuleshov2014algorithms} also observed that the ordinary (non-sliding) $\epsilon$-greedy policy performed well compared to many other algorithms, however in the setting of stationary rewards and a smaller number of arms. It is interesting to note that even in this more difficult setting of non-stationary payouts and $K=10^3$ arms the sliding $\epsilon$-greedy policy significantly outperforms random sampling, and 
matches the performance of an adaptive method designed to model time dependence. 
As we increase the amount of noise $\sigma_{\rm noise}$ the performance of all non-random policies performance degrades somewhat, but the Quick-Draw still outperforms the rest and with less severe performance degradation.

In Fig.~\ref{fig:results_field_power} we show the results of the bandit policies varying the power $\alpha$ to which we raise the mean payout function $\mu(x, t)$. As $\alpha$ is increased, the sharpness of the payout function is increased, exaggerating the difference in payouts between optimal and sub-optimal arms. This makes the payout environment a more challenging problem. 
We see that the performance of random sampling degrades because the `volume' of metric space occupied by optimal arms is reduced.  
We also see that the Quick-Draw policy again outperforms all baseline policies for all values of the field power.

\begin{figure}[!h]
\centering
  \centering
  \includegraphics[width=\defaultfigwidth]{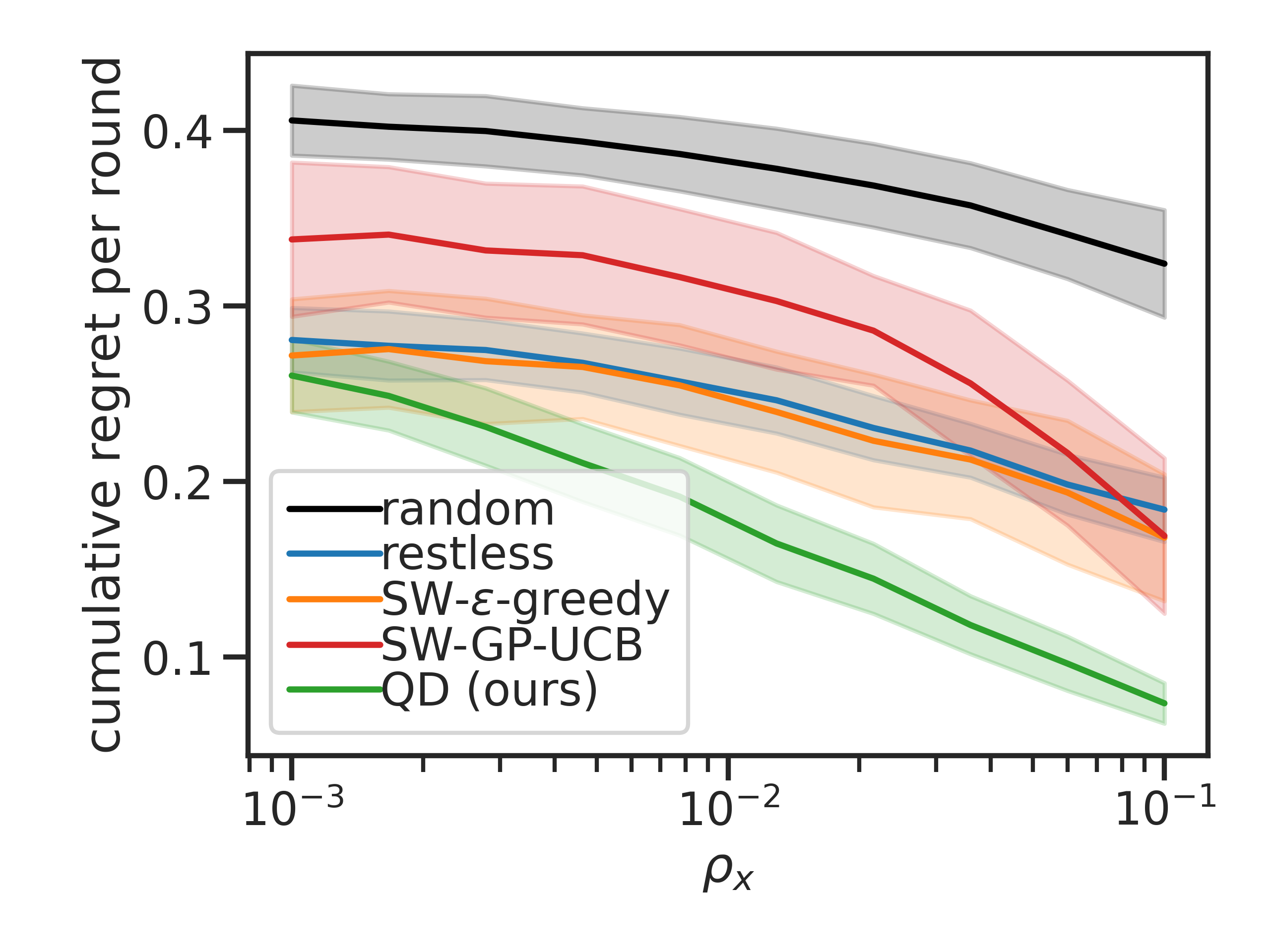}
  \caption{Comparison of the cumulative regret per round for each bandit policy as the spatial correlation length $\rho_x$ of the mean payout functions is varied.}
  \label{fig:results_rho_x}
\end{figure}

\begin{figure}[!h]
\centering
  \centering
  \includegraphics[width=\defaultfigwidth]{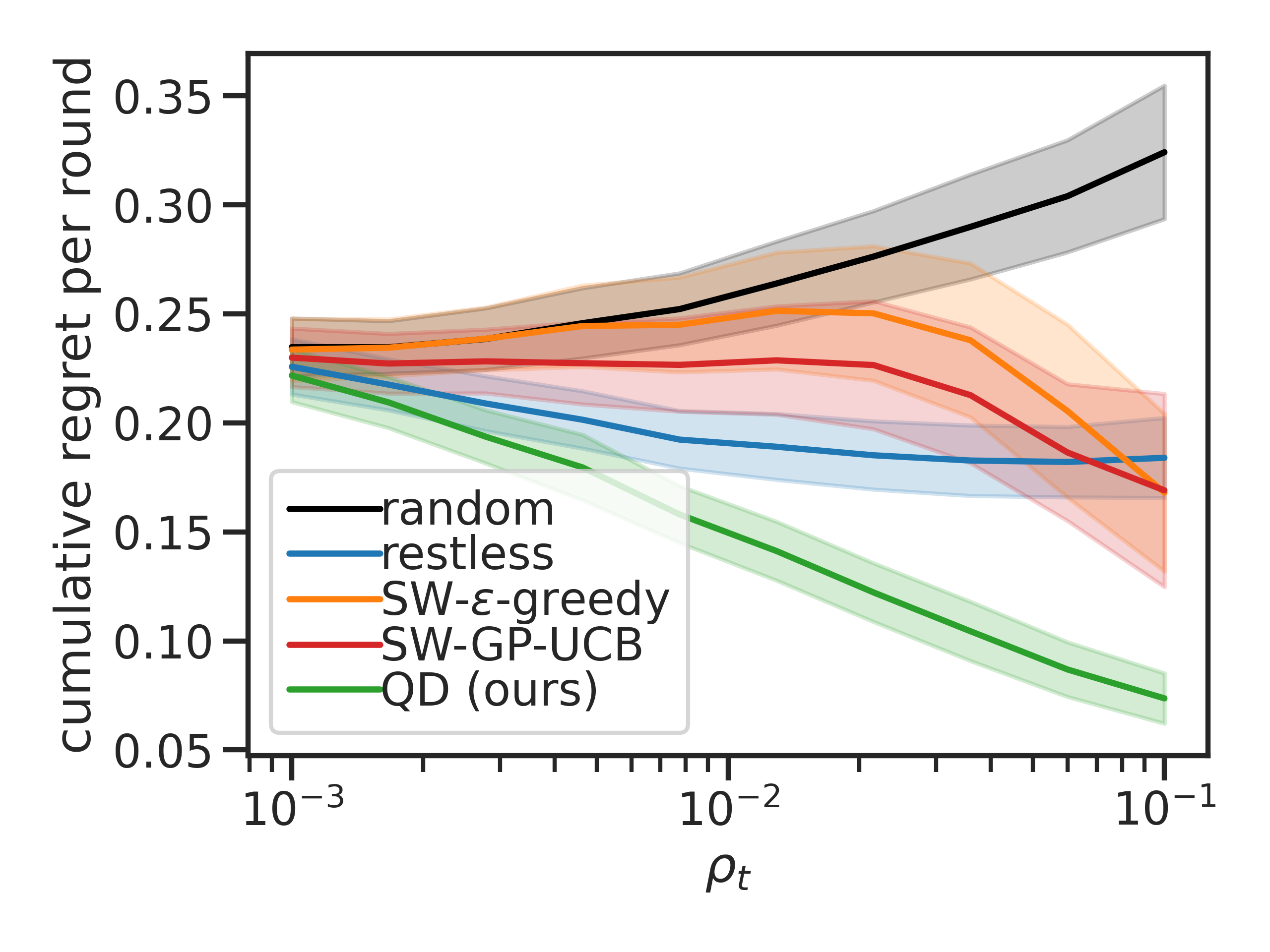}
  \caption{Comparison of the cumulative regret per round for each bandit policy as the temporal correlation length $\rho_t$ of the mean payout functions is varied.}
  \label{fig:results_rho_t}
\end{figure}

In Fig.~\ref{fig:results_rho_x} we compare the cumulative regrets of the policies as the spatial correlation length $\rho_x$ of the field is varied. Because we keep the arms at a fixed separation distance $\Delta x = 2 / K = 2 \cdot 10^{-3}$, reducing $\rho_x$ makes the bandit problem more difficult. When $\rho_x = 10^{-3}$, the Quick-Draw policies performance degrades to be very similar to both $\epsilon$-greedy and restless bandit performance, while significantly outperforming the SW-GP-UCB policy. In this case, while no structure in $x$ can be exploited, structure in time $t$ remains at length scales $ < \rho_t$. The windowed $\epsilon$-greedy, restless bandit and Quick-Draw policies use this information to outperform random sampling.
However, as the metric space correlation length $\rho_x$ increases, the Quick-Draw policy outperforms the rest.

Additionally, in Fig.~\ref{fig:results_rho_t} we compare the performance of the policies as the payout field temporal correlation length $l_t$ is varied. When $\rho_t = 10^{-3}$, all policies degrade to random sampling performance. In this extreme situation, where $\tau_s = \rho_t$,
the bandit is observing payouts in time that retain no correlation/no usable information between successive rounds, and thus we should expect any policy to perform only as well as random sampling. However, even when $\tau_s / \rho_t = 10$, meaning that only roughly the last ten rounds are relevant concerning the current payouts, the restless bandit policy significantly performs all other baselines.

\begin{figure}[!h]
\centering
  \centering
  \includegraphics[width=\defaultfigwidth]{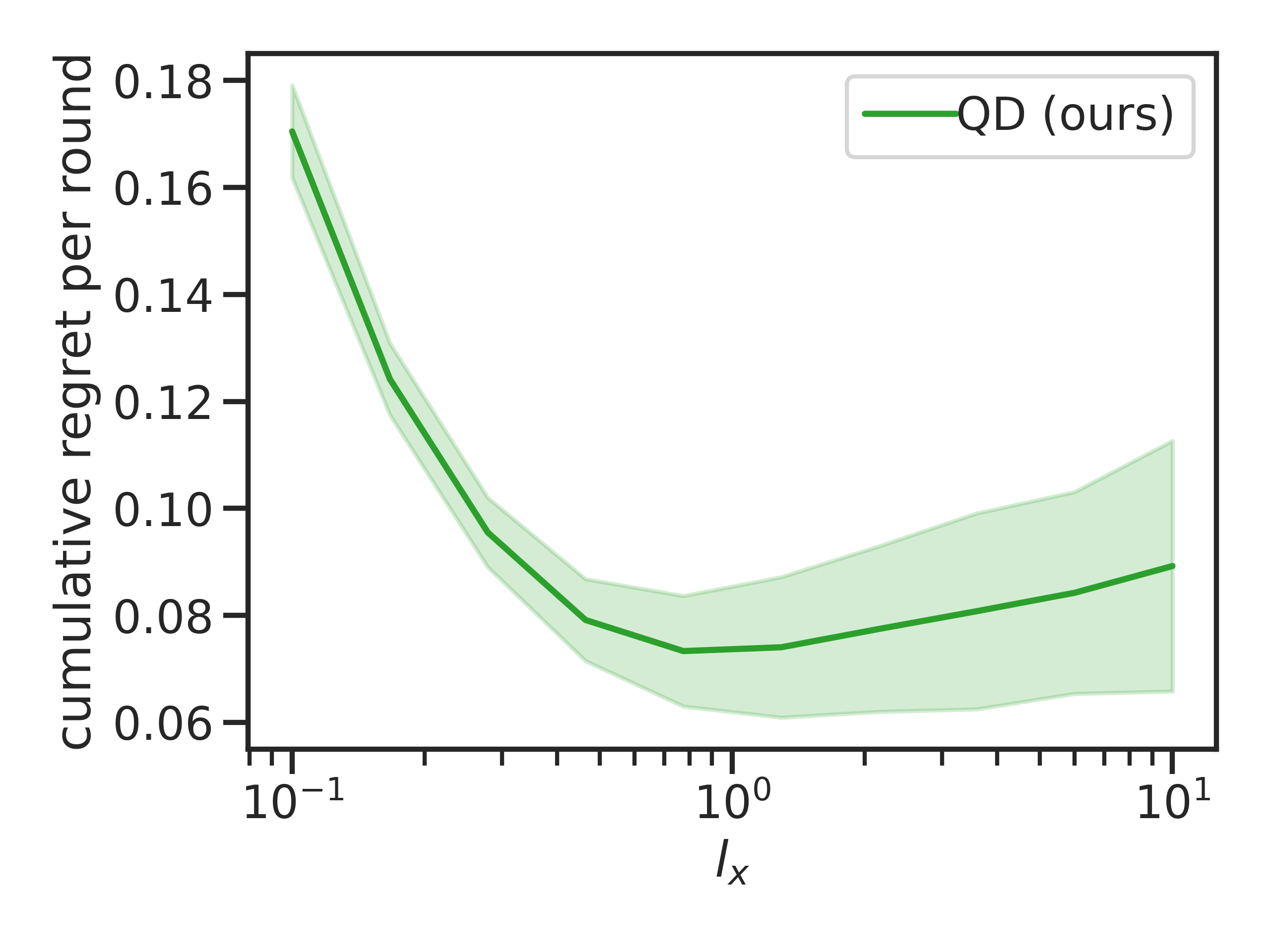}
  \caption{The cumulative regret per round for the Quick-Draw bandit policy as the policy parameter $l_x$ is varied.}
  \label{fig:results_l_x}
\end{figure}

\begin{figure}[!h]
\centering
  \centering
  \includegraphics[width=\defaultfigwidth]{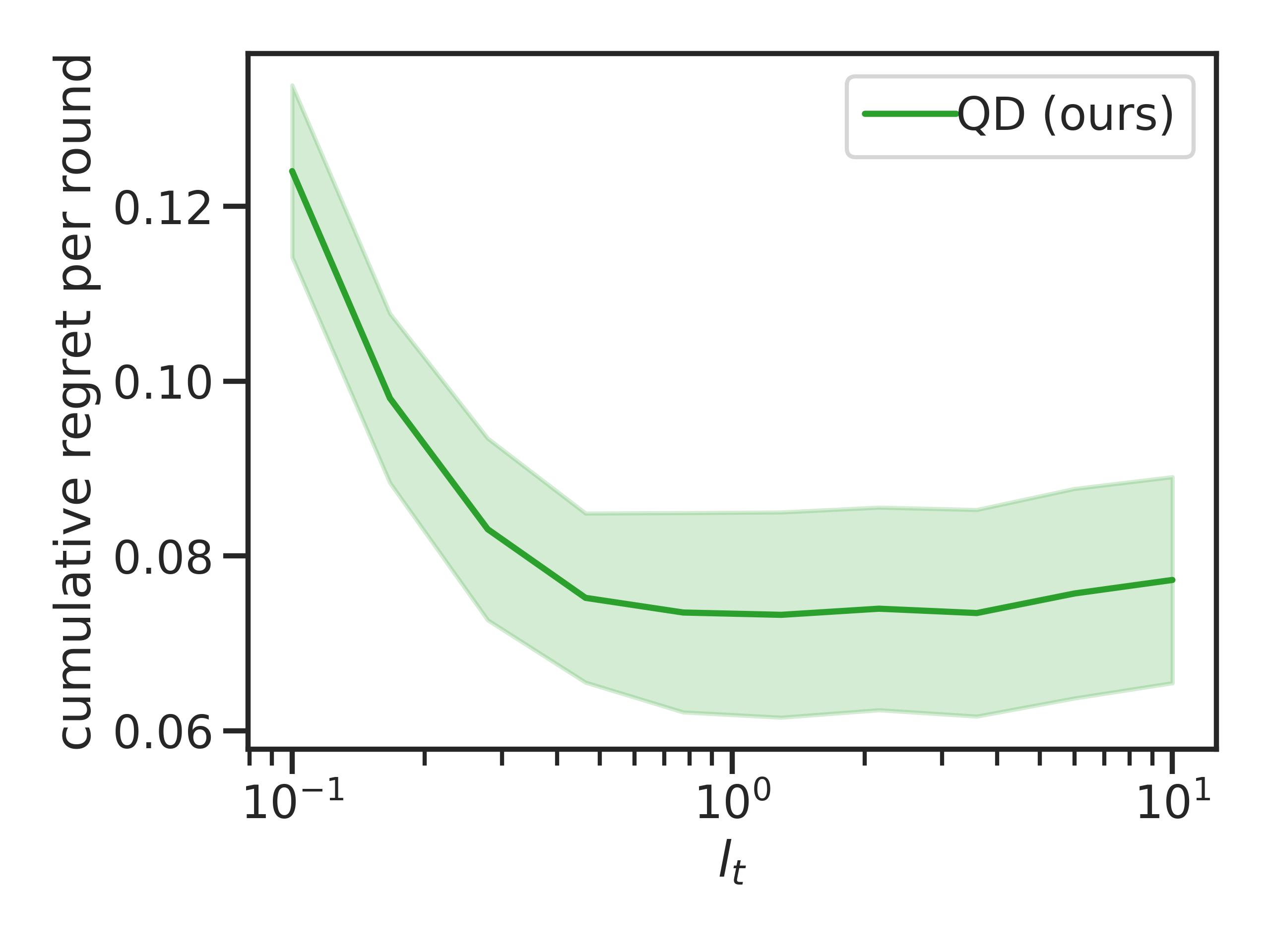}
  \caption{The cumulative regret per round for the Quick-Draw bandit policy as the policy parameter $l_t$ is varied.}
  \label{fig:results_l_t}
\end{figure}

Finally, in Figs.~\ref{fig:results_l_x} and ~\ref{fig:results_l_t}
we display the performance of the Quick-Draw policy as a function of its parameters $\ell_x$ and $\ell_t$, respectively. We see that the performance degrades as the order of magnitude of either $l_x$ or $l_t$ is decreased below $1$.
As a reminder, these parameters enter the policy through Eqn.~\ref{eqn:sigma_hat}, being in the denominators of each term. 
Thus, when either $\ell_x$ or $\ell_t$ are decreased, $\hat{\sigma}$ increases. This leads to an increase of $\hat{\Sigma}_T$, and therefore
a widening of the credible interval.
In this case, the policy increases its exploration and reduces its exploitation.

\begin{figure}[h!]
\centering
  \centering
  \includegraphics[width=\defaultfigwidth]{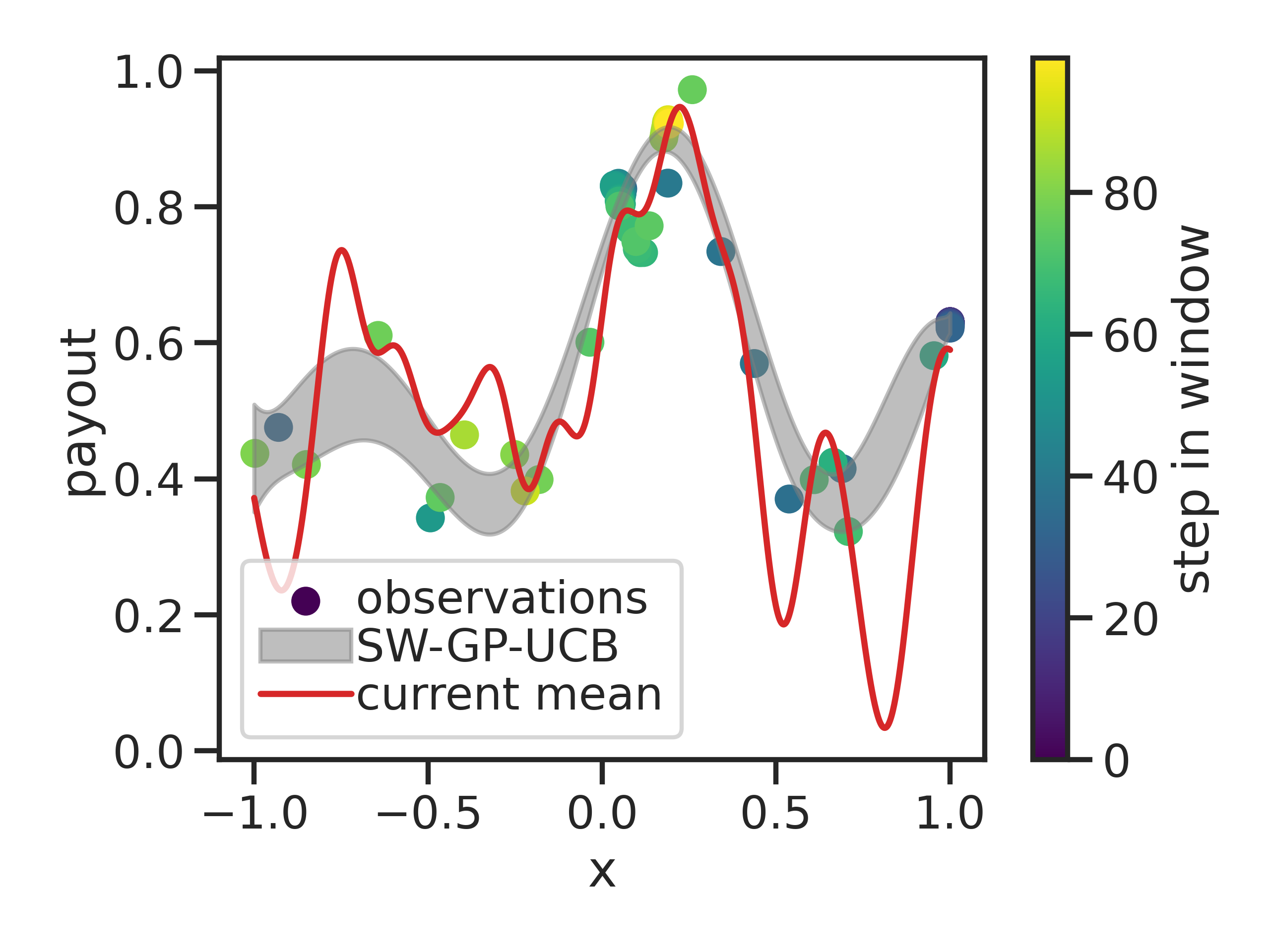}
  \caption{The SW-GP-UCB policy can sometimes behave well. The color of observed samples corresponds to their time/round within the observation window, with the most recent samples shown in yellow. The GP prediction for the upper-confidence-bound of the payout is shown in gray, and the true/oracle current mean payout is in red.}
  \label{fig:sw_gp_ucb_step_220}
\end{figure}

\begin{figure}[!h]
\centering
  \centering
  \includegraphics[width=\defaultfigwidth]{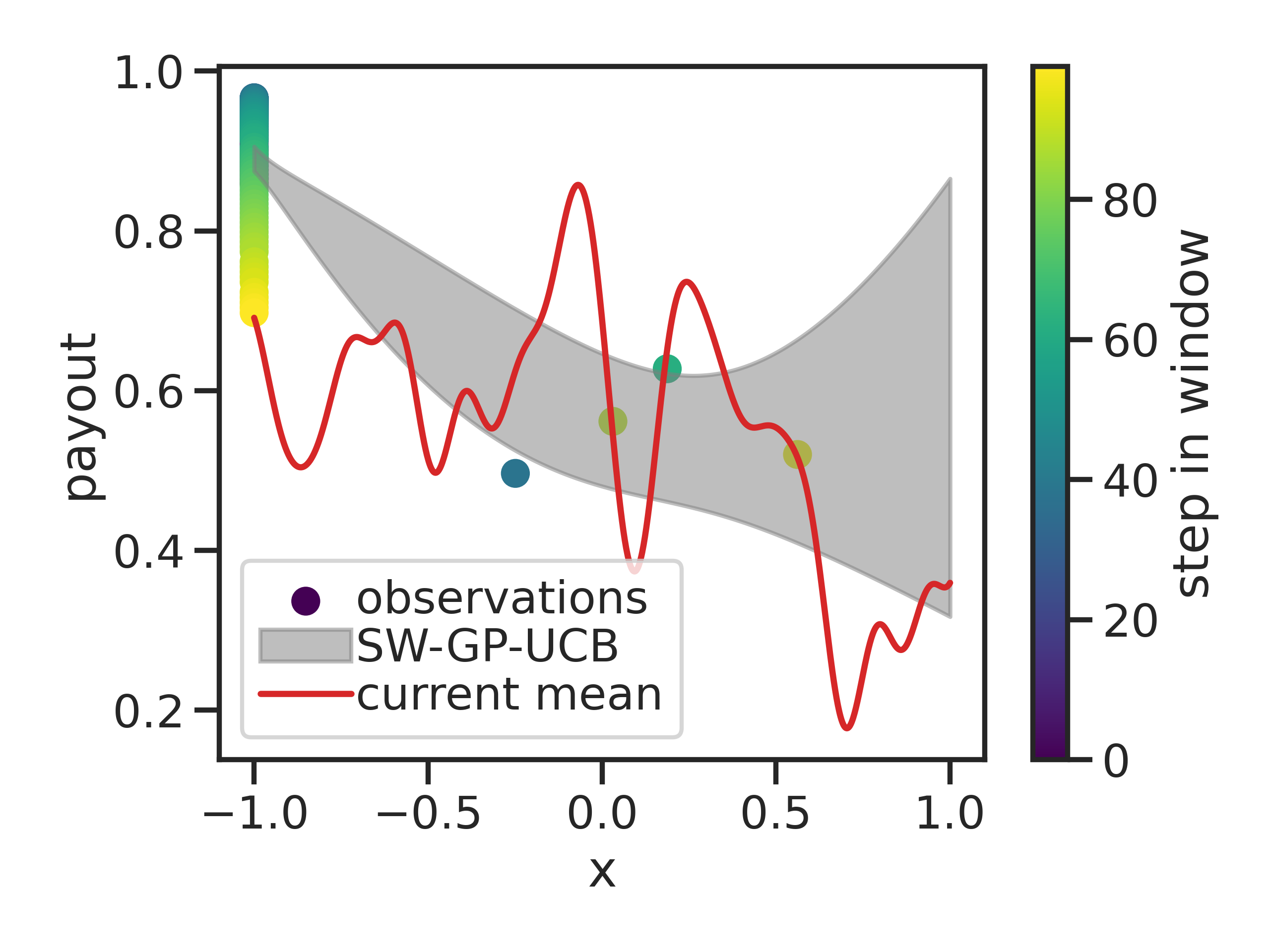}
  \caption{The SW-GP-UCB policy often fails, which causes edge-seeking behavior. The color of observed samples corresponds to their time/round within the observation window, with the most recent samples shown in yellow. The GP prediction for the upper-confidence-bound of the payout is shown in gray, and the true/oracle current mean payout is in red.}
  \label{fig:sw_gp_ucb_step_520}
\end{figure}

We illustrate the performance of the SW-GP-UCB.
In Fig.~\ref{fig:sw_gp_ucb_step_220}, we see that in some cases, the Gaussian Process has sufficient data within its observation window to reasonably describe the current mean payouts. 
However, as seen in Fig.~\ref{fig:sw_gp_ucb_step_520},
it often enters a failure mode in which the Gaussian Process is overconfident with respect to the current mean payout, and too heavily exploits arms at the edge. This edge-seeking continues even as the arms at the edge become highly suboptimal.

\section{Evaluation on Open Bandit Dataset}
\label{sec:open_bandit}

In this section, we evaluate our novel Quick-Draw algorithm against baselines on a public real-world dataset. The Open Bandit Dataset contains logs of an A/B test for two baseline algorithms, random arm selection and Bernoulli Thompson sampling~\cite{saito2020large}. We performed our evaluation on the `women' subset of data. The events collected with the random arm selection policy total $864,585$ events. There are 46 arms, each corresponding to a unique product shown on the user's webpage. Each event records whether the user clicked on the link to the displayed product, and this defines the reward. The relevant quantity that a bandit policy aims to maximize in this setup is the click-through rate (CTR). 

The features describing the user in the event data provide a context with potentially useful information regarding the optimal product to display. Therefore, we segmented the data according to unique combinations of categorical user features. Then bandit policies were played and evaluated on each user feature group independently, in the time order of their collection.   
We took the product numerical feature and rescaled it to $[-1, 1]$, and this feature corresponded to the metric space for our Quick-Draw model. Additionally, the total time interval for data collected was rescaled to $[0, 1]$. Additionally, we broke the event data, segmented by user feature group, into separate intervals of one thousand rounds each. 

We compare different bandit policies via off-policy evaluation using Inverse Propensity Scoring (IPS) (also called Inverse Probability Weighting) ~\cite{horvitz1952generalization}. 
The IPS-estimated expected reward for a \emph{target} policy $\pi_t$ under data collected according to a \emph{logging} policy $\pi_l$ is given by
\beq
\hat{V}_{\rm IPS}(\pi_t | \mathcal{D}) = \mathbb{E}_{\mathcal{D}}\left[ \frac{\pi_t(a|\xi)}{\pi_l(a|\xi)}  r(a|\xi) \right].
\eeq

The data $\mathcal{D} = (\xi, a, r(a|\xi))$ are collected according to the actions $a$ of the logging policy, given the context $\xi$, upon which it receives a reward $r$. 
$\pi_t(a|\xi)$ denotes the probability (`propensity') that the target policy would choose action $a$ given context $\xi$, and likewise for $\pi_l(a|\xi)$.

\begin{table}[]
\centering
\begin{tabular}{@{}lc@{}}
\toprule
target policy            & click-through rate (\%) \\ \midrule
random            & $0.49 \pm 0.05$         \\
SW-GP-UCB         & $0.57 \pm 0.01$         \\
restless bandit   & $0.98 \pm 0.02$         \\
SW-$\epsilon$-greedy & $2.12 \pm 0.04$         \\
\textbf{Quick-Draw} & $3.51 \pm 0.04$         \\ \bottomrule
\end{tabular}
\caption{Click-through rates (higher is better) for each MAB policy estimated via inverse propensity scoring on Open Bandit Dataset.
}
\label{tab:ctr}
\end{table}

The estimated click-through rates of the bandit policies are displayed in Table ~\ref{tab:ctr}. These mean and standard deviation estimates were obtained by running each policy for ten independent trials. 
The logging policy for the data collected was itself a random policy. We also include the results of running a random \emph{target} policy through our inverse propensity estimator as a validation of the method: the ground-truth click-through rate for the events collected by the random \emph{logging} policy is $0.48\%$, which matches our estimate to within statistical uncertainty.

We see that the $\epsilon$-greedy policy (which exploits a leading arm 90\% of the rounds on average), restless bandit, and Quick-Draw (with $\ell_t = \ell_x = 1$) policies  
significantly outperforms random selection. 
The SW-GP-UCB policy performs barely better than random selection; this is likely because the Gaussian Process model for the observations is a poor fit to Bernoulli distributed observations. 
Additionally, the Quick-Draw algorithm outperforms both the $\epsilon$-greedy and restless bandit policies. 
The Open Bandit Dataset data present a bandit problem which is in between the extremes. $K=46$ arms should be sufficient to begin to challenge ordinary stochastic bandit algorithms, but not large enough that we should expect those methods to completely break down. Even in this intermediate regime (for which public datasets can be found), the advantages of modeling the time and feature-space correlations via our Quick-Draw policy bear measurable improvements over baseline policies.

\begin{figure}[h!]
\centering
 \includegraphics[width=\defaultfigwidth]{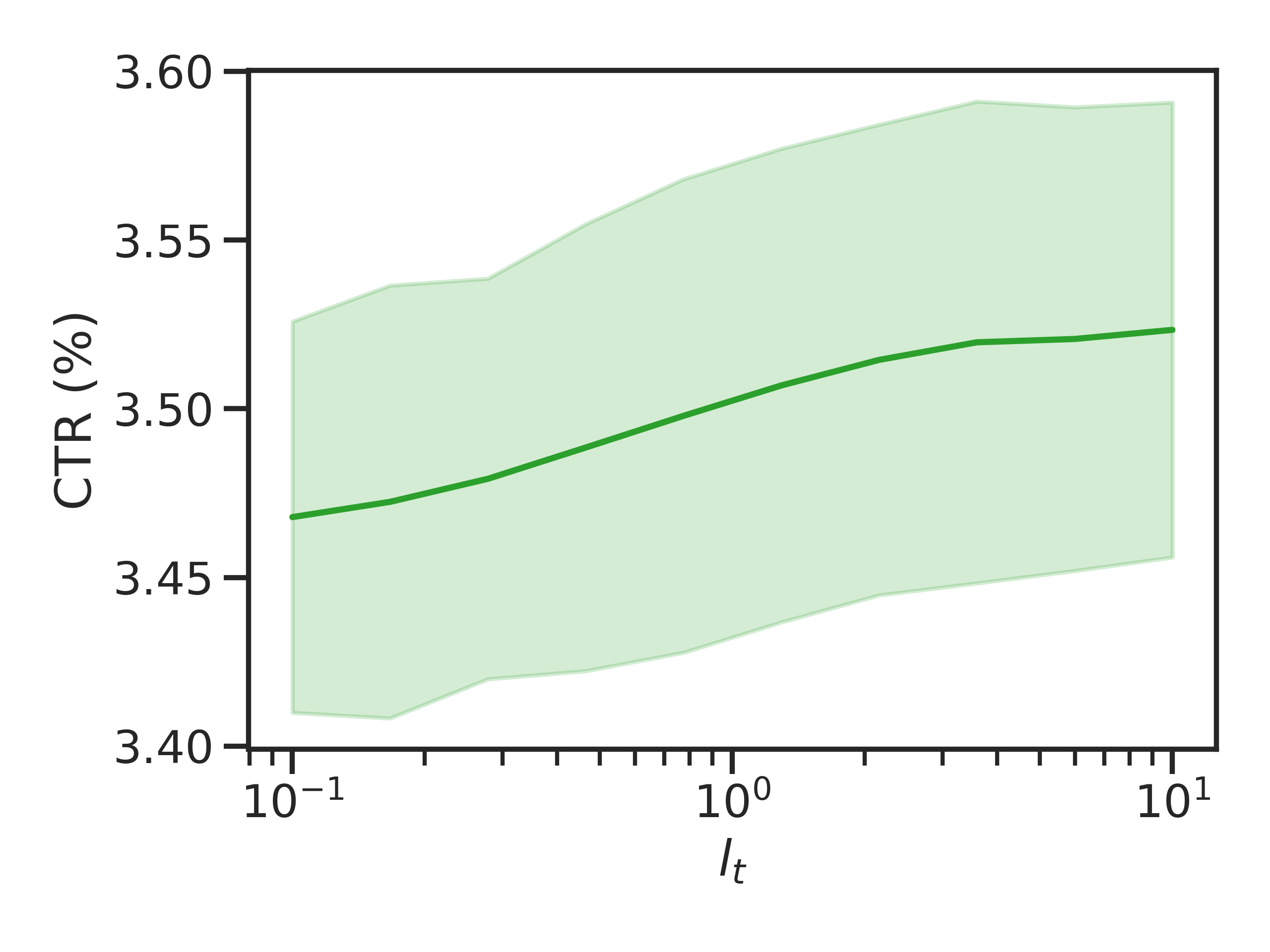}
\caption{The Click-Through Rate (CTR) for the Quick-Draw bandit policy on Open Bandit Dataset as the policy parameter $l_t$ is varied. }
\label{fig:ctr_vs_lt}
\end{figure}

Additionally, in Fig.~\ref{fig:ctr_vs_lt} we display the performance of the Quick-Draw policy on the Open Bandit dataset as the policy parameter $\ell_t$ varies. We see that the performance of the policy is stable, and its performance has only minor sensitivity to $\ell_t$.

\section{Conclusion}
\label{sec:conclusion}

In this manuscript, we have motivated and developed the Quick-Draw bandit policy, a multi-armed bandit policy that is capable of tackling extreme payout environments for which existing baseline methods break down. The Quick-Draw policy performs significantly better than baseline methods in situations in which the number of bandit arms $K$ is greater than the number of rounds for which previously observed payout information is relevant. Additionally, we have shown that the Quick-Draw policy outperforms baseline policies on a real-world public dataset composed of click-through responses to suggested products.

\bibliographystyle{acm}
\balance 
\bibliography{biblio}

\appendix

\section{Policy exploration/exploitation via hyperparameters}

\begin{figure}[!h]
  \centering
  \includegraphics[width=\defaultfigwidth]{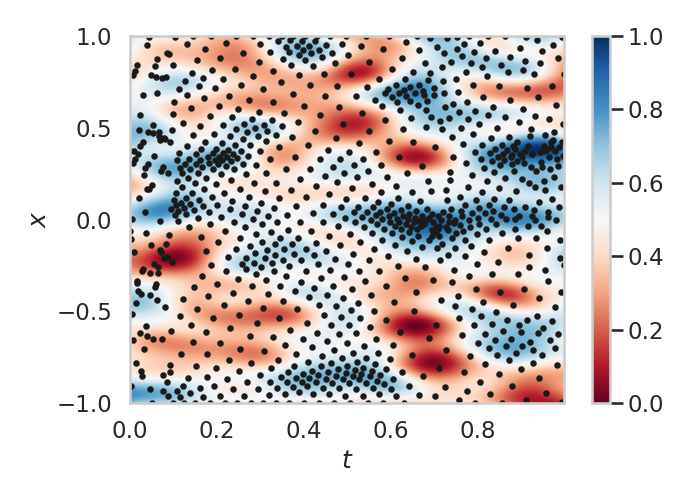}
  \caption{The behavior of the Quick-Draw policy for $\ell_x = \ell_t = 0.1$. The first $100$ rounds ($0 \leq t \leq 0.1$) are fixed to random sampling.}
  \label{fig:mean_field_dyn_lip_explore}
\end{figure}

The behavior of the Quick-Draw policy is displayed in Fig.~\ref{fig:mean_field_dyn_lip_explore} in a case where the hyperparameters are small. 
When $\ell_x$ or $\ell_t$ are too small, the policy explores too much at the expense of increased accumulated regret.

\section{Proofs}

\subsection{Proof of Theorem~\ref{thm:concentration}}

\begin{dummythm}
Suppose the Quick-Draw algorithm is run for up to $T_{\max}$ rounds,
and set hyperparameter $\ell_x$ small enough so that $\ell_x\leq 1/\sqrt{T_{\max} - \rho^2}$.
For any $\delta > 0$, set $\gamma_T \coloneqq 2 L + 4 C_1 \ln^2(2T^2/\delta)$.
Then for all $x\in\mathcal{X}$ and all $1\leq T\leq T_{\max}$, 
    \begin{equation}
        | \mu(x) - \hat\mu_T(x)| \leq \gamma_{T+1}\hat\Sigma_T
    \end{equation}
    holds with probability at least $1 - \delta$.
\end{dummythm}

\begin{proof}
First consider a specific iteration $T$ of the algorithm, with observed points $y_1, \ldots, y_T$ measured at points $x_1, \ldots, x_T$. Let $\mu_1, \ldots, \mu_T$ represent their underlying mean rewards. Our goal for this iteration is to bound the tail probability of the event
\begin{equation}
    \sup_{x\in \mathcal{X}}|\mu(x) - \hat\mu_T(x)| > \gamma_{T+1} \hat\Sigma_T
\end{equation}

For convenience define the precision function of each observation $\nu_t(x) \coloneqq (\hat\sigma^2_t(x))^{-1} = ( \rho^2 + D(x, x_t)^2/\ell_x^2)^{-1}$ and the normalized precision $\alpha_t(x) \coloneqq \nu_t(x) / \sum_{s\in[T]} \nu_s(x)$. Observe that the mean estimator can be rewritten as
\begin{align*}
    \hat\mu_T(x) &= \hat\Sigma_T(x) \cdot\sum_{s=[T]} \frac{y_s}{\hat\sigma_s^2(x)} \\
    &= \Big(\sum_{s} \frac{1}{\hat\sigma_s^2(x)}\Big)^{-1}\cdot \sum_{s} \frac{y_s}{\hat\sigma_s^2(x)}\\
    &= \Big(\sum_{s} \nu_s(x) \Big)^{-1} \cdot \sum_{s} \nu_s(x) y_s\\
    &= \sum_s \alpha_s y_s
\end{align*}
dropping the dependence on $x$ for simplicity. We have so far shown that $\hat\mu_T$ is a convex combination of the observed rewards $\{y_s\}_{s\in [T]}$ with corresponding weights $\alpha_s$. 

Choose any $x\in \mathcal{X}$.
First suppose that $\mu(x) \geq \hat\mu_T(x)$.
Applying Lemma~\ref{lem:alpha_lip_bound},
\begin{align*}
    |\mu(x) - \hat\mu_T(x)| &= \mu(x) - \sum_{s\in [T]} \alpha_s y_s \\
    &\leq  \sum_{s\in [T]} \alpha_s ( \mu_s + L \cdot D(x, x_s) - y_s)\\
    &= \sum_{s\in [T]} \alpha_s ( \varepsilon_s + L \cdot D(x, x_s))\\
    &= \sum_{s\in [T]} \alpha_s \varepsilon_s + L \sum_{s\in[T]}\alpha_s D(x, x_s)
\end{align*}

Similarly, when $\mu(x) < \hat\mu_T(x)$, the lower bound gives
\begin{align*}
    |\mu(x) - \hat\mu_T(x) | &= \hat\mu_T(x) - \mu(x) \\
    &\leq -\sum_{s\in[T]} \alpha_s \varepsilon_s +  L \sum_{s\in[T]} \alpha_s D(x, x_s)
\end{align*}
and we can convince ourselves that $|\mu(x) - \hat\mu_T(x)| \leq \big|\sum_{s\in[T]} \alpha_s \varepsilon_s\big| + L\sum_{s\in[T]} \alpha_s D(x, x_s)$. We now need to show that the quantity
\begin{equation} \label{eq:thm1_norm_error}
    \frac{|\mu(x) - \hat\mu_T(x)|}{ \hat\Sigma_T} \leq \frac{  \big| \sum_{s} \alpha_s \varepsilon_s \big| }{ \hat\Sigma_T} + L \frac{ \sum_{s}\alpha_s D(x, x_s) }{\hat\Sigma_T}
\end{equation}
is bounded by $\gamma_{T+1}$ with high probability.

We start with bounding the second quantity. Disregarding $L$ for now, by definition we first obtain:
\begin{align*}
    \frac{\sum_{s} \alpha_s D(x, x_s)}{ \hat\Sigma_T} &= \frac{\sum_s \nu_s D(x, x_s)}{\sum_s \nu_s \cdot \hat\Sigma_T}\\
    &= \frac{\sum_s \nu_s D(x, x_s) \cdot \sqrt{\sum_s \nu_s}}{\sum_s \nu_s} \\
    &= \frac{\sum_s \nu_s D(x, x_s)}{\sqrt{\sum_s \nu_s}}
\end{align*}

By Cauchy-Schwarz, the numerator is
\begin{align*}
    \sum_s \nu_s D(x, x_s) &= \sqrt{\Big( \sum_s \nu_s D(x, x_s)\Big)^2  } \\
    &\leq \sqrt{ \sum_s \nu_s \cdot \sum_s \nu_s D(x, x_s)^2}
\end{align*}
which makes the total quantity of interest now
\begin{equation} \label{eq:lipschitz_deviation}
    \frac{\sum_s \alpha_s D(x, x_s)}{\hat\Sigma_T} \leq \sqrt{\sum_s \nu_s D(x, x_s)^2} \eqqcolon \phi_T
\end{equation}

Further note that $\phi_T \leq \sqrt{T \max_t \frac{D(x, x_t)^2}{\rho^2 + D(x, x_t)^2/\ell_x^2}} \leq \sqrt{\frac{T}{\rho^2 + 1/\ell_x^2}}$ since $\nu_t$ is an increasing function of $D(x, x_t)$. This can be verified by taking the derivative.
Suppose we wish to control $\phi_T = 
\mathcal{O}(1)$. Solving for $\ell_x$, we get $$\ell_x = \mathcal{O}\Bigg(\sqrt{\frac{1}{T - \rho^2}}\Bigg)$$
Given $T_{\max}$ rounds, we can thus set $\ell_x = \mathcal{O}(1/\sqrt{T_{\max} - \rho^2})$ which suffices to completely control Eq.~\ref{eq:lipschitz_deviation}. 
That is, for small enough $\ell_x$, we can set the quantity to be bounded by $L$.
Observing that this final quantity is $x$-independent and holds over all $x\in\mathcal{X}$ simultaneously, we note that this is a uniform bound over $x$. 


Next, we need to stochastically bound the first term of Eq.~\ref{eq:thm1_norm_error}.
Suppose $T$ is fixed and consider the sequence of prior iterations $t\leq T$.
Conditional on the history up to and including $(x_{t}, y_{t})$, $\alpha_t$ and $\hat\Sigma_t$ are fixed, while $\varepsilon_s$ has mean zero.
We can then define $M_t \coloneqq \alpha_t \varepsilon_t / \hat\Sigma_t$ to be a martingale difference sequence.
Using Freedman's concentration bound for martingale difference sequences, we show in Lemma~\ref{lem:mds_bound} that for all $x\in\mathcal{X}$, with probability $1-\delta$, 
$$ \Big| \sum_{s=[T]} M_s \Big| \leq \gamma_{T+1}/2$$
for all $T\geq 1$.
Here, $\gamma_T$ is defined as
\begin{equation} \label{eq:gamma_t}
    \gamma_{T} \coloneqq 2 L + 4C_1\ln^2(2T^2/\delta)
\end{equation}

Now that we have bounded both quantities in Eq.~\ref{eq:thm1_norm_error}, we show the original statement.
Suppose the high-probability event holds. Then, for all $1\leq T\leq T_{\max}$, and for all $x\in\mathcal{X}$,
\begin{align*}
    \frac{|\mu(x) - \hat\mu_{T}(x)| }{\hat\Sigma_T}   &\leq \Big|\sum_{s} M_s \Big| + L  \\
    &\leq \gamma_{T+1}/2 + L \\
    &= 2L + 2C_1\ln^2(2(T+1)^2/\delta)\\
    &\leq \gamma_{T+1}
\end{align*}

This concludes the proof.

\end{proof}

%
\begin{lemma} \label{lem:alpha_lip_bound}
    Take any set of observed $x_1, \ldots, x_T$ with means $\mu_1, \ldots, \mu_T$. Let $\{\alpha_s\}_{s=1}^T$ be any set of weights such that $\alpha_s \geq 0$ and $\sum_{s=1}^T \alpha_s = 1$. Then $\forall x$, $$\sum_{s\in [T]} \alpha_s (\mu_s - L\cdot D(x, x_s)) \leq \mu(x) \leq \sum_{s\in [T]} \alpha_s (\mu_s + L\cdot D(x, x_s))$$
\end{lemma}
\begin{proof}
Fix $x$. From the Lipschitz property, $|\mu(x) - \mu_t | \leq L \cdot D(x, x_t) $ for any $x_t, \mu_t$.
From this we see that $\mu(x)$ is constrained as $$\max_{t \in [T]} \{\mu_t - L \cdot D(x, x_t)\} \leq \mu(x) \leq \min_{t \in [T]} \{ \mu_t + L \cdot D(x, x_t)\} $$

\noindent Since for any $s\in [T]$, $\alpha_s \cdot \min_{t\in [T]} \{\mu_t + L\cdot D(x, x_t)\} \leq \alpha_s (\mu_s + L \cdot D(x, x_s))$, we further deduce that 
$$ \mu(x) \leq \min_{t \in [T]} \{ \mu_t + L \cdot D(x, x_t)\}  \leq \sum_{s\in [T]} \alpha_s (\mu_s + L\cdot D(x, x_s))$$
Using the same argument for the lower bound, \\$\mu(x) \geq \sum_{s\in [T]} \alpha_s (\mu_s - L\cdot D(x, x_s))$, completing the proof.
\end{proof}

\begin{lemma}\label{lem:bound_on_nu}
    Assume the stationary case and let $\nu_t \coloneqq (\hat\sigma_t^2)^{-1}$. Then $\frac{1}{\rho^2 + 1/\ell_x^2} \leq \nu_t \leq \frac{1}{\rho^2}$.
\end{lemma}
\begin{proof}
In the stationary case we assume the time-dependent component of variance is constant and included in $\rho^2$. Then by definition $\nu_t = \frac{1}{\rho^2 + D(x, x_t)^2 /\ell_x^2}$. Since $0 \leq D(x, x_t) \leq 1$, then
$\frac{1}{\rho^2 + 1/\ell_x^2} \leq \nu_t \leq \frac{1}{\rho^2}$.
\end{proof}

\begin{lemma} \label{lem:mds_bound}
    Let $\gamma_{T}$ be defined as in Eq.~\ref{eq:gamma_t}. Then, for all $x\in\mathcal{X}$, the martingale difference sequence defined as $M_t \coloneqq \alpha_t \varepsilon_t / \hat\Sigma_t$ satisfies
    $$P\Big(\forall T\geq 1, \bigg|\sum_{s=1}^T M_s \bigg| \leq \gamma_{T+1}/2\Big) \geq 1-\delta$$
\end{lemma}
\begin{proof}
    This result leverages Freedman's concentration inequality for martingale difference sequences (Lemma~\ref{lem:freedman}).
    First fix any $T>0$.
    We must bound the magnitude and the conditional variance of the sequence.
    By definition, we have

\begin{align*}
    |M_t| &=  \bigg| \frac{\alpha_t \varepsilon_t}{\hat\Sigma_t}\bigg| \\
    &= \frac{\nu_t |\varepsilon_t| \sqrt{\sum_{s=1}^t \nu_s}}{\sum_{s=1}^t \nu_s} \\
    &= \frac{\nu_t |\varepsilon_t|}{\sqrt{\sum_{s=1}^t \nu_s}}\\
    &\leq \frac{|\varepsilon_t|}{\rho^2 \sqrt{\sum_{s} 1/(\rho^2 + 1/\ell_x^2)}}\\
    &= \frac{\sqrt{\rho^2 + 1/\ell_x^2}}{\rho^2 \sqrt{t}}\\
    &= \frac{C_1}{\sqrt{t}}
\end{align*}
where the final inequality uses Lemma~\ref{lem:bound_on_nu} and $|\varepsilon| \leq 1$.

Next the conditional variance of $M_t$
is bounded
$\mathbb{E}[M_t^2] \leq \frac{C_1^2}{t}$.
Then the sum of conditional variances $V$ is
\begin{align*}
    V &= \sum_{s\in[T]} \mathbb{E}[M_s^2] \\
    &\leq \sum_{s=[T]} \frac{C_1^2}{s}\\
    &\leq C_1^2 (\ln T + 1)
\end{align*}

In particular, we note that the two upper bounds above are independent of $x$. That is, they hold uniformly for any set of $\alpha_t(x)$. Now we can apply Lemma~\ref{lem:freedman} to get, for all $x \in \mathcal{X}$, 
\begin{multline*}
    P\Big(\ \Big|\sum_{s\in[T]}M_s \Big| \geq \gamma_{T+1}/2\Big) \\
    = P\Big(\ \Big|\sum_{s\in[T]} M_s \Big| \geq \gamma_{T+1}/2 \ \mathrm{and}\ V \leq C_1^2 (\ln T + 1)\Big)\\
    \leq 2 P\Big(\sum_{s\in[T]} M_s \geq \gamma_{T+1}/2 \ \mathrm{and}\ V \leq C_1^2 (\ln T + 1)\Big)\\
    \leq 2\cdot \exp\Bigg(\frac{-(\gamma_{T+1}/2)^2}{2C_1^2(\ln T + 1) + \frac{(2/3) C_1 (\gamma_{T+1}/2)}{\sqrt{T}}}\Bigg)\\
    = 2\cdot \exp\Bigg(\frac{-\gamma_{T+1}^2 \sqrt{T}}{8C_1^2\sqrt{T}(\ln T + 1) + (4/3)C_1 \gamma_{T+1}}\Bigg)\\
    \leq 2 \cdot \max \Bigg\{ \exp\Bigg( \frac{-\gamma_{T+1}^2}{16 C_1^2 (\ln{T} + 1)} \Bigg), \exp\Bigg( \frac{-3\gamma_{T+1} \sqrt{T}}{8C_1}  \Bigg) \Bigg\}
\end{multline*}
where the first inequality uses the union bound and symmetry of $M_s$, and the second bound considers the cases where each component of the denominator is larger than the other. Now we aim to convert this to a bound which holds with $1-\delta$ probability. Setting each expression in the last line to $\delta/T^2$, we get that $\gamma_{T+1}$ must be at least
$\max\bigg\{ \sqrt{16C_1^2(\ln{T} + 1) \ln(2T^2/\delta)}, (8/3)C_1\ln(2T^2/\delta) / \sqrt{T} \bigg\}$.

\newpage

From inspection we can deduce that \\ $4 C_1 \sqrt{(\ln T + 1)(\ln(2T^2/\delta)} \ln(2T^2/\delta) \leq 4C_1 \ln^2(2T^2/\delta) \leq \gamma_{T+1}$ is sufficiently large to satisfy this bound.

We have so far shown that for a given $T$, \\ $P\Big( \ \Big|\sum_{s\in[T]}M_s \Big| \geq \gamma_{T+1}/2\Big) \leq \delta/T^2$.
By the union bound,
\begin{align*}
     P\Big(\bigcup_{T\geq 1}  \Big|\sum_{s\in[T]}  M_s \Big| \geq \gamma_{T+1}/2 \Big) &\leq \sum_{T\geq 1} P\Big( \ \Big|\sum_{s\in[T]} M_s \Big| \geq \gamma_{T+1}/2\Big) \\
     &= 0 + \sum_{T\geq 2} \delta/T^2\\
     &= \delta(\pi^2/6 - 1) \leq \delta
\end{align*}
noting that $|M_1| \leq \gamma_2/2$ with probability 1. Therefore
\begin{equation}
    P\Big( \forall T\geq 1,\ \Big|\sum_{s\in[T]}M_s \Big| \leq \gamma_{T+1}/2\Big) \geq 1-\delta
\end{equation}
completing the proof.
\end{proof}

\begin{lemma}[Freedman]\label{lem:freedman}
    Suppose $X_1, \ldots, X_T$ is a martingale difference sequence with filtration $\mathcal{F}_1 \subseteq \ldots\subseteq \mathcal{F}_T$ (that is, $\mathbb{E}[X_t | \mathcal{F}_{t-1}] = 0$), and suppose $|X_t| \leq b$ for all $t$. Let $V = \sum_{t=1}^T \mathrm{Var}(X_t | \mathcal{F}_{t-1})$ be the sum of conditional variances. Then for all $a, v > 0$,
    $$P\bigg(\sum_{t=1}^T X_t \geq a\ \mathrm{and}\ V \leq v\bigg) \leq \exp\bigg(\frac{-a^2}{2v + 2ab/3}\bigg)$$
\end{lemma}
\begin{proof}
    See~\cite{freedman1975tail} for the proof.
\end{proof}

\subsection{Proof of regret bounds}

\textbf{Proof of Lemma~\ref{lem:ucb_regret_per_round}.}
\begin{dummylemma}
Choose any $t\leq T_{\max}$. Under the condition $\{ |\mu(x) - \hat\mu_t(x) | \leq \gamma_{t+1} \hat\Sigma_t(x) \}$, for all $x\in\mathcal{X}$, then the regret at round $t$
satisfies
$$ r_t \leq 2 \gamma_{t+1} \hat\Sigma_t(x_t) $$
where $x_t$ is the point selected by the UCB policy.
\end{dummylemma}

\begin{proof}
    This result is well-known in UCB literature, see \cite{slivkins2019introduction} section 1.3, or \cite{srinivas2009gaussian} Lemma 5.2. 
\end{proof}

\textbf{Proof of Theorem~\ref{thm:total_regret}.}

\begin{dummythm}
    Let $\delta > 0$. Run the Quick-Draw algorithm for $T_{\max}$ iterations where $T_{\max}$ can be arbitrarily large.
    If $\gamma_T$ and $\ell_X$ are set as before,
    then with probability $1-\delta$,
    the cumulative regret is
    $$R_T \leq 4CL + 8C^2\sqrt{T}\ln^2(2T^2/\delta)$$
    where $C\coloneqq \sqrt{\rho^2 + 1/\ell_x^2}$.
    
    This gives the asymptotic regret
$$R_T = \mathcal{O}\Big(C^2 \sqrt{T} \ln^2(2T^2/\delta)\Big)$$
\end{dummythm}

\begin{proof}
    By definition and applying Lemma~\ref{lem:ucb_regret_per_round}, the cumulative regret is
    \begin{align*}
        R_T = \sum_{t=[T]} r_t &\leq 2 \sum_{t\in[T]} \gamma_{t+1} \hat\Sigma_t\\
        &\leq 2 \gamma_{T+1} \sum_{t} \hat\Sigma_t
    \end{align*}

    \begin{align*}
    \hat\Sigma_t &= \bigg(\sum_{s\in[t]} \frac{1}{\rho^2 + D(x,x_s)^2/\ell_x^2}  \bigg)^{-1/2} \\
    &\leq \bigg(  \sum_{s} \frac{1}{\rho^2 + 1/\ell_x^2} \bigg)^{-1/2} \\
    &= \bigg(\frac{t}{\rho^2 + 1/\ell_x^2} \bigg)^{-1/2}\\
    &= \frac{\sqrt{\rho^2 + 1/\ell_x^2}}{\sqrt{t}}
\end{align*}

Since $\sqrt{t} \leq \sqrt{T}$ for any $t\leq T$, we have that \\ $\sum_t \hat\Sigma_t \leq (1/\sqrt{T})\sum_t\sqrt{\rho^2 + 1/\ell_x^2} = \sqrt{T(\rho^2 + 1/\ell_x^2)}$.
Then
\begin{align*}
    R_T &\leq (4L + 8C_1 \ln^2(2T^2/\delta)) \sqrt{T(\rho^2 + 1/\ell_x^2)}\\
    &= 4\rho^2C_1L + 8\rho^2 C_1^2 \sqrt{T}\ln^2(2T^2/\delta)
\end{align*}

Importantly, if we redefine $C\coloneqq \rho^2 C_1 = \sqrt{\rho^2 + 1/\ell_x^2}$, then the asymptotic regret is
$$R_T = \mathcal{O}\Big(C^2 \sqrt{T} \ln^2(2T^2/\delta)\Big)$$
\end{proof}

\end{document}